\pdfoutput=1

\documentclass[letterpaper]{article}
\usepackage{uai18}
\usepackage[margin=1in]{geometry}

\usepackage{times}
\usepackage{microtype}
\usepackage{graphicx}
\usepackage{subfigure}
\usepackage{booktabs} 
\usepackage{hyperref}       
\usepackage{url}            
\usepackage{booktabs}       
\usepackage{amsfonts}       
\usepackage{nicefrac}       
\usepackage{microtype}      
\usepackage{listings}
\usepackage{natbib}
\usepackage{amsthm}
\usepackage{amsmath}
\usepackage{color}
\usepackage{float}
\usepackage{bbm}
\usepackage{todonotes}
\usepackage{bm}
\usepackage{multirow}
\usepackage[switch]{lineno} 

\hypersetup{
    colorlinks=true,
    linkcolor=black,
    citecolor=black,
    filecolor=black,
    urlcolor=black,
}

\allowdisplaybreaks

\newtheorem{lem}{\textbf{Lemma}}
\newcommand{\E}{\mathbb{E}}
\newcommand{\Sv}{$\mathcal{S}$}
\newcommand{\Nv}{$\mathcal{N}$}
\newcommand{\x}{\mathbf{x}}
\newcommand{\y}{\mathbf{y}}
\newcommand{\z}{\mathbf{z}}
\newcommand{\e}{\mathbf{e}}
\newcommand{\uv}{\mathbf{u}}
\newcommand{\vv}{\mathbf{v}}

\title{Hyperspherical Variational Auto-Encoders}

\author{ {\bf Tim R. Davidson\thanks{\;\;Equal contribution. . Correspondence to: Nicola De Cao $<$\href{mailto:nicola.decao@gmail.com}{nicola.decao@gmail.com}$>$.}}
\quad
{\bf Luca Falorsi\footnotemark[1]}
\quad
{\bf Nicola De Cao\footnotemark[1]}
\quad
{\bf Thomas Kipf}
\quad
{\bf Jakub M. Tomczak}\\\\
University of Amsterdam
}

\begin{document}

\maketitle

\begin{abstract}
The Variational Auto-Encoder (VAE) is one of the most used unsupervised machine learning models. But although the default choice of a Gaussian distribution for both the prior and posterior represents a mathematically convenient distribution often leading to competitive results, we show that this parameterization fails to model data with a latent hyperspherical structure. To address this issue we propose using a von Mises-Fisher (vMF) distribution instead, leading to a hyperspherical latent space. Through a series of experiments we show how such a hyperspherical VAE, or  $\mathcal{S}$-VAE, is more suitable for capturing data with a hyperspherical latent structure, while outperforming a normal,  $\mathcal{N}$-VAE, in low dimensions on other data types.
\end{abstract}

\section{INTRODUCTION}

First introduced by \cite{journals/corr/KingmaW13, rezende2014stochastic}, the Variational Auto-Encoder (VAE) is an unsupervised generative model that presents a principled fashion for performing variational inference using an auto-encoding architecture. Applying the non-centered parameterization of the variational posterior \citep{KingmaWelling-2014-eff-gradient-inf}, further simplifies sampling and allows to reduce bias in calculating gradients for training. Although the default choice of a Gaussian prior is mathematically convenient, we can show through a simple example that in some cases it breaks the assumption of an \textit{uninformative} prior leading to unstable results. Imagine a dataset on the circle $\mathcal{Z} \subset \mathcal{S}^1$, that is subsequently embedded in $\mathbb{R}^N$ using a transformation $f$ to obtain $f: \mathcal{Z} \to \mathcal{X} \subset \mathbb{R}^N$. Given two hidden units, an autoencoder quickly discovers the latent circle, while a normal VAE becomes highly unstable. This is to be expected as a Gaussian prior is concentrated around the origin, while the KL-divergence tries to reconcile the differences between $\mathcal{S}^1$ and $\mathbb{R}^2$. 

The fact that some data types like \textit{directional data} are better explained through spherical representations is long known and well-documented \citep{mardia1975statistics, fisher1987statistical}, with examples spanning from protein structure, to observed wind directions. Moreover, for many modern problems such as text analysis or image classification, data is often first normalized in a preprocessing step to focus on the directional distribution. Yet, few machine learning methods explicitly account for the intrinsically spherical nature of some data in the modeling process. In this paper, we propose to use the \textit{von Mises-Fisher} (vMF) distribution as an alternative to the Gaussian distribution. This replacement leads to a hyperspherical latent space as opposed to a hyperplanar one, where the Uniform distribution on the hypersphere is conveniently recovered as a special case of the vMF. Hence this approach allows for a truly uninformative prior, and has a clear advantage in the case of data with a hyperspherical interpretation. This was previously attempted by \cite{hasnat2017mises}, but crucially they do not learn the concentration parameter around the mean, $\kappa$.

In order to enable training of the concentration parameter, we extend the \textit{reparameterization trick} for rejection sampling as recently outlined in \cite{rejection-repar} to allow for $n$ additional transformations. We then combine this with the rejection sampling procedure proposed by \cite{sample-vmf} to efficiently reparameterize the VAE \footnote{Code freely  available on: \url{https://github.com/nicola-decao/s-vae}}.

We demonstrate the utility of replacing the normal distribution with the von Mises-Fisher distribution for generating latent representations by conducting a range of experiments in three distinct settings. First, we show that our \Sv-VAEs outperform VAEs with the Gaussian variational posterior (\Nv-VAEs) in recovering a hyperspherical latent structure. Second, we conduct a thorough comparison with \Nv-VAEs on the MNIST dataset through an unsupervised learning task and a semi-supervised learning scenario. Finally, we show that \Sv-VAEs can significantly improve link prediction performance on citation network datasets in combination with a \textit{Variational Graph Auto-Encoder} (VGAE) \citep{kipf2016VGAE}. 

\begin{figure*}[t!]
\centering
     \subfigure[Original]{\includegraphics[width=0.19\textwidth]{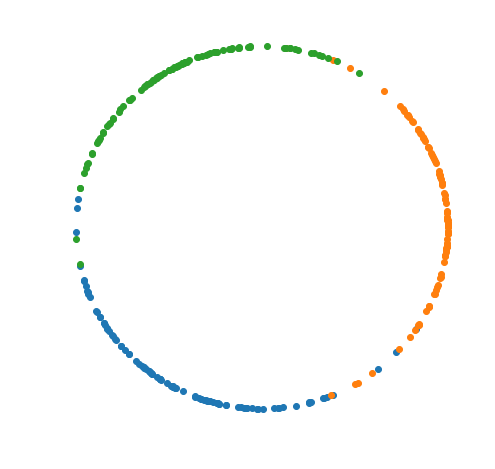} \label{fig:toy1_orig}}
     \subfigure[Autoencoder]{\includegraphics[width=0.19\textwidth]{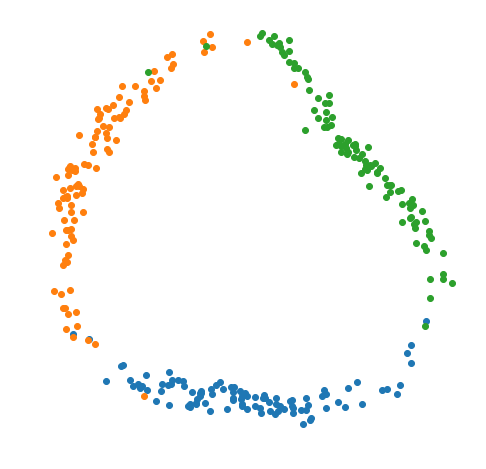} \label{fig:toy1_AE}}
     \subfigure[$\mathcal{N}$-VAE]{\includegraphics[width=0.19\textwidth]{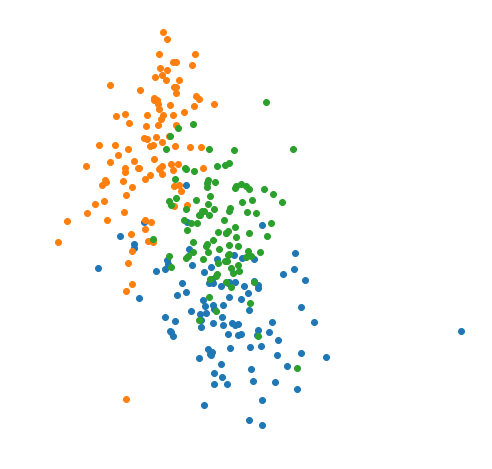} \label{fig:toy1_VAE}}
     \subfigure[$\mathcal{N}$-VAE, $\beta=0.1$]{\includegraphics[width=0.19\textwidth]{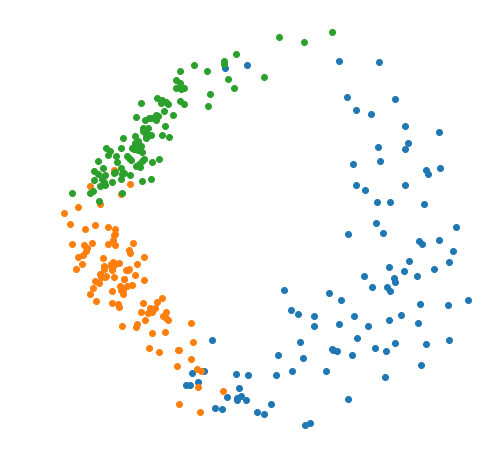} \label{fig:toy1_VAE2}}
     \subfigure[$\mathcal{S}$-VAE]{\includegraphics[width=0.19\textwidth]{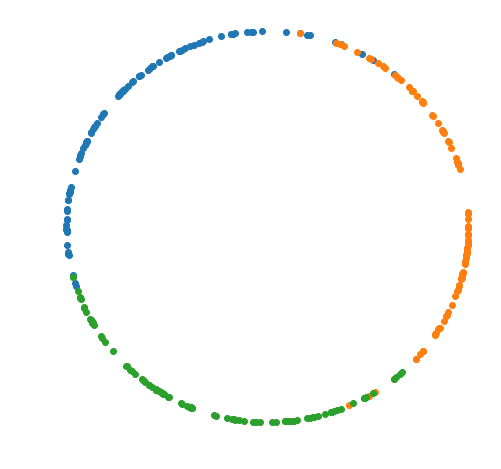}
     \label{fig:toy1_SVAE}}
\caption{Plots of the original latent space (a) and learned latent space representations in different settings, where $\beta$ is a re-scaling factor for weighting the KL divergence. (Best viewed in color)}
\label{fig:exp-toy1}
\end{figure*}

\section{VARIATIONAL AUTO-ENCODERS}
\label{sec:vae}

\subsection{FORMULATION} \label{subsec:formulation}
In the VAE setting, we have a latent variable model for data, where $\z \in \mathbb{R}^{M}$ denotes latent variables, $\x$ is a vector of $D$ observed variables, and $p_{\phi}(\x,\z)$ is a parameterized model of the joint distribution. Our objective is to optimize the log-likelihood of the data, $\log \int p_{\phi}(\x,\z) d\z$. When $p_{\phi}(\x,\z)$ is parameterized by a neural network, marginalizing over the latent variables is generally intractable. One way of solving this issue is to maximize the Evidence Lower Bound (ELBO)
\begin{align}
\log \int p_{\phi}(\x,\z) d\z \ge \; \E_{q(\z)}&[\log p_{\phi}(\x|\z)] + \nonumber \\
&- KL(q(\z)||p(\z)),
\end{align}
where $q(\z)$ is the approximate posterior distribution, belonging to a family $\mathcal{Q}$. The bound is tight if $q(\z) = p(\z|\x)$, meaning $q(\z)$ is optimized to approximate the true posterior. While in theory $q(\z)$ should be optimized for every data point $\x$, to make inference more scalable to larger datasets the VAE setting introduces an inference network $q_{\psi}(\z|\x;\theta)$ parameterized by a neural network that outputs a probability distribution for each data point $\x$. The final objective is therefore to maximize
\begin{align}
\mathcal{L}(\phi,\psi) =\; \E_{q_{\psi}(\z|\x;\theta)}&[\log p_{\phi}(\x|\z)] + \nonumber \\
&- KL(q_{\psi}(\z|\x;\theta)||p(\z)),
\end{align}
In the original VAE both the prior and the posterior are defined as normal distributions. We can further efficiently approximate the ELBO by Monte Carlo estimates, using the \textit{reparameterization trick} \citep{journals/corr/KingmaW13, rezende2014stochastic}. This is done by expressing a sample of $\z\sim q_{\psi}(\z|\x;\theta)$, as $\z = h(\theta,\varepsilon,\x)$, where $h$ is a reparameterization transformation and $\varepsilon\sim s(\varepsilon)$ is some noise random variable independent from $\theta$.

\subsection{THE LIMITATIONS OF A GAUSSIAN DISTRIBUTION PRIOR} \label{subsec:lim-gaussian-prior}

\paragraph{Low dimensions: origin gravity} \label{par:lowdim}
In low dimensions, the Gaussian density presents a concentrated probability mass around the origin, encouraging points to cluster in the center. This is particularly problematic when the data is divided into multiple clusters. Although an ideal latent space should separate clusters for each class, the normal prior will encourage all the cluster centers towards the origin. An ideal prior would only stimulate the variance of the posterior without forcing its mean to be close to the center. A prior satisfying these properties is a uniform over the entire space. Such a uniform prior, however, is not well defined on the hyperplane.
 
\paragraph{High dimensions: soap bubble effect} \label{par:highdim}
It is a well-known phenomenon that the standard Gaussian distribution in high dimensions tends to resemble a uniform distribution on the surface of a hypersphere, with the vast majority of its mass concentrated on the hyperspherical shell. Hence it would appear interesting to compare the behavior of a Gaussian approximate posterior with an approximate posterior already naturally defined on the hypersphere. This is also motivated from a theoretical point of view, since the Gaussian definition is based on the $L_2$ norm that suffers from the curse of dimensionality.

\subsection{BEYOND THE HYPERPLANE} \label{subsec:beyond-hyperplanes}
Once we let go of the hyperplanar assumption, the possibility of a uniform prior on the hypersphere opens up. Mirroring our discussion in the previous subsection, such a prior would exhibit no pull towards the origin allowing clusters of data to evenly spread over the surface with no directional bias. Additionally, in higher dimensions, the cosine similarity is a more meaningful distance measure than the Euclidean norm.
 
\paragraph{Manifold mapping}\label{par:manifold}
In general, exploring VAE models that allow a mapping to distributions in a latent space not homeomorphic to $\mathbb{R}^D$ is of fundamental interest. Consider data lying in a small $M$-dimensional manifold $\mathcal{M}$, embedded in a much higher dimensional space $\mathcal{X} = \mathbb{R}^N$. For most real data, this manifold will likely not be homeomorphic to $\mathbb{R}^{M}$. An encoder can be considered as a smooth map $enc : \mathcal{X} \to \mathcal{Z}  = \mathbb{R}^{D}$ from the original space to $\mathcal{Z}$. The restriction of the encoder to $\mathcal{M}$, $enc|_\mathcal{M} : \mathcal{M} \to \mathcal{Z}$ will also be a smooth mapping. However since $\mathcal{M}$ is not homeomorphic to $\mathcal{Z}$ if $D \le M$, then $enc|_\mathcal{M}$ cannot be a homeomorphism. That is, there exists no invertible and globally continuous mapping between the coordinates of $\mathcal{M}$ and the ones of $\mathcal{Z}$. Conversely if $D > M$ then $\mathcal{M}$ can be smoothly embedded in $\mathcal{Z}$ for $D$ sufficiently big
\footnote{By the Whitney embedding theorem any smooth real $M$-dimensional manifold can be smoothly embedded in $\mathbb{R}^{2M}$}
, such that $enc|_\mathcal{M}:\ \mathcal{M} \to enc|_\mathcal{M}(\mathcal{M})=: emb(\mathcal{M})\subset\mathcal{Z}$ is a homeomorphism and $emb(\mathcal{M})$ denotes the embedding of $\mathcal{M}$. Yet, since $D > M$, when taking random points in the latent space they will most likely \textit{not} be in $emb(\mathcal{M})$ resulting in a poorly reconstructed sample. 
 
The VAE tries to solve this problem by forcing $\mathcal{M}$ to be mapped into an approximate posterior distribution that has support in the entire $\mathcal{Z}$. Clearly, this approach is bound to fail since the two spaces have a fundamentally different structure. This can likely produce two behaviors: first, the VAE could just smooth the original embedding $emb(\mathcal{M})$ leaving most of the latent space empty, leading to bad samples. Second, if we increase the KL term the encoder will be pushed to occupy all the latent space, but this will create instability and discontinuity, affecting the convergence of the model. To validate our intuition we performed a small proof of concept experiment using $\mathcal{M} = \mathcal{S}^1$, which is visualized in Figure \ref{fig:exp-toy1}. Note that as expected the auto-encoder in Figure \ref{fig:toy1_AE} mostly recovers the original latent space of Figure\ref{fig:toy1_orig} as there are no distributional restrictions. In Figure \ref{fig:toy1_VAE} we clearly observe for the \Nv-VAE that points collapse around the origin due to the KL, which is much less pronounced in Figure \ref{fig:toy1_VAE2} when its contribution is scaled down. Lastly, the \Sv-VAE almost perfectly recovers the original circular latent space. The observed behavior confirms our intuition. 

To solve this problem the best option would be to directly specify a $\mathcal{Z}$ homeomorphic to $\mathcal{M}$ and distributions on $\mathcal{M}$. However, for real data discovering the structure of $\mathcal{M}$ will often be a difficult inference task. Nevertheless, we believe this shows that investigating VAE architectures that map to posterior distributions defined on manifolds different than the Euclidean space is a topic worth to be explored. In that sense, this work represents an initial step in this research direction. 

\section{REPLACING GAUSSIAN WITH VON MISES-FISHER}
\subsection{VON MISES-FISHER DISTRIBUTION}\label{subsec:vMF}

The \textit{von Mises-Fisher} (vMF) distribution is often described as the Normal Gaussian distribution on a hypersphere. Analogous to a Gaussian, it is parameterized by $\mu\in\mathbb{R}^{m}$ indicating the mean direction, and $\kappa\in\mathbb{R}_{\geq0}$ the concentration around $\mu$. For the special case of $\kappa = 0$, the vMF represents a Uniform distribution. The probability density function of the vMF distribution for a random unit vector $\z \in \mathbb{R}^{m}$ (or $\z \in \mathcal{S}^{m-1}$) is then defined as
\begin{align}
    q(\z|\mathbf{\mu}, \kappa) &=  \mathcal{C}_m(\kappa)\exp{(\kappa \mathbf{\mu}^T\z)} \label{vMF-def}\\
    \mathcal{C}_m(\kappa) &= \dfrac{\kappa^{m/2 - 1}}{(2\pi)^{m/2}\mathcal{I}_{m/2 - 1}(\kappa)}, \label{eq-norm-c}
\end{align}
where $||\mathbf{\mu}||^2 = 1$, $\mathcal{C}_m(\kappa)$ is the normalizing constant, and $\mathcal{I}_v$ denotes the modified Bessel function of the first kind at order $v$.

\subsection{KL DIVERGENCE} \label{subsec:kl}
As previously emphasized, one of the main advantages of using the vMF distribution as an approximate posterior is that we are able to place a uniform prior on the latent space. The KL divergence term $KL(\text{vMF}(\mu, \kappa)|| U(S^{m-1}))$ to be optimized is: 
\begin{align} \label{eq:kl}
\kappa \frac{\mathcal{I}_{m/2}(k)}{\mathcal{I}_{m/2 - 1}(k)} + \log \mathcal{C}_m(\kappa) - \log \left( \frac{2(\pi^{m/2})}{\Gamma(m/2)} \right)^{-1},
\end{align}
see Appendix \ref{ap:kl-divergence} for complete derivation. Notice that since the KL term does not depend on $\mu$, this is only optimized in the reconstruction term. The above expression cannot be handled by automatic differentiation packages because of the modified Bessel function in $\mathcal{C}_m(\kappa)$. Thus, to optimize this term we derive the gradient with respect to the concentration parameter $\nabla_{\kappa}KL(\text{vMF}(\mu, \kappa)|| U(S^{m-1}))$:
\begin{align}\label{eq:kl-grad}
\frac12 k \Biggl(& \frac{\mathcal{I}_{m/2+1}(k)}{\mathcal{I}_{m/2 - 1}(k)} +
\nonumber \\
&- \frac{\mathcal{I}_{m/2}(k) \left( \mathcal{I}_{m/2 - 2}(k) + \mathcal{I}_{m/2}(k) \right) }{\mathcal{I}_{m/2 - 1}(k)^2} + 1 \Biggl),
\end{align}
where the modified Bessel functions can be computed without numerical instabilities using the exponentially scaled modified Bessel function.

\subsection{SAMPLING PROCEDURE}\label{sec:sampling}

\begin{algorithm}[tb]\label{sampling-vmf}
   \caption{vMF sampling}
   \label{alg:vmf-sample}
\begin{algorithmic}
    \STATE {\bfseries Input:} dimension $m$, mean $\mu$, concentration $\kappa$
    \STATE sample $\vv \sim U(\mathcal{S}^{m-2})$
    \STATE sample $\omega \sim g(\omega| \kappa,m) \propto \exp(\omega \kappa)(1-\omega^2)^{\frac{1}{2}(m-3)}$ \COMMENT{acceptance-rejection sampling}
    \STATE $\z'\gets (\omega; (\sqrt{1-\omega^2}) \vv^\top )^\top$
    \STATE $U\gets Householder(\e_1,\mu)$ \COMMENT{Householder transform}
    \STATE{\bfseries Return:}{ $U\z'$}
\end{algorithmic}
\end{algorithm}

To sample from the vMF we follow the procedure of \citet{sample-vmf}, outlined in Algorithm \ref{alg:vmf-sample}. We first sample from a vMF $q(\z| \e_1 , \kappa)$ with modal vector $\e_1 = (1, 0, \cdots, 0)$. Since the vMF density is uniform in all the $m-2$ dimensional sub-hyperspheres $\{\x \in \mathcal{S}^{m-1}| \e_1^\top \x = \omega\}$, the sampling technique reduces to sampling the value $\omega$ from the univariate density $g(\omega|\kappa,m)\propto \exp(\kappa \omega)(1-\omega^2)^{(m-3)/2}, \quad \omega\in[-1,1]$, using an acceptance-rejection scheme. After getting a sample from  $q(\z| \e_1 , \kappa)$  an orthogonal transformation $U(\mu)$ is applied such that the transformed sample is distributed according to $q(\z|\mu , \kappa)$.
This can be achieved using a Householder reflection such that $U(\mu)\e_1 = \mu$. A more in-depth explanation of the sampling technique can be found in Appendix \ref{app:vmf-sample}. 

It is worth noting that the sampling technique does not suffer from the curse of dimensionality, as the acceptance-rejection procedure is only applied to a univariate distribution. Moreover in the case of $\mathcal{S}^2$, the density $g(\omega|\kappa,3)$ reduces to $g(\omega|\kappa,3)\propto \exp(k\omega)\mathbbm{1}_{[-1,+1]}(\omega)$ which can be directly sampled without rejection.

\subsection{N-TRANSFORMATION REPARAMETERIZATION TRICK}\label{sec:reparameterization}

While the \textit{reparameterization trick} is easily implementable in the normal case, unfortunately it can only be applied to a handful of distributions. However a recent technique introduced by \cite{rejection-repar} allows to extend the reparameterization trick to the wide class of distributions that can be simulated using rejection sampling. Dropping the dependence from $\x$ for simplicity, assume the approximate posterior is of the form $g(\omega| \theta)$ and that it can be sampled by making proposals from $r(\omega| \theta)$. If {\it the proposal distribution can be reparameterized} we can still perform the reparameterization trick. Let $\varepsilon \sim s(\varepsilon)$, and $\omega = h(\varepsilon,\theta)$, a reparameterization of the proposal distribution, $r(\omega|\theta)$. Performing the reparameterization trick for $g(\omega| \theta)$ is made possible by the fundamental lemma proven in \citep{rejection-repar}:
\begin{lem}\label{normal-repar}
Let $f$ be any measurable function and $\varepsilon \sim \pi(\varepsilon| \theta) = s(\varepsilon)\dfrac{g(h(\varepsilon,\theta)| \theta)}{r(h(\varepsilon,\theta)| \theta)}$ the distribution of the accepted sample.  Then:
\begin{align}
    \E_{ \pi(\varepsilon| \theta)} &[f(h(\varepsilon, \theta))] = \int f(h(\varepsilon, \theta))\pi(\varepsilon| \theta) d\varepsilon \nonumber \\
    &= \int f(\omega)g(\omega|\theta)d\omega = \E_{g(\omega|\theta)}[f(\omega)],
\end{align}
\end{lem}
Then the gradient can be taken using the log derivative trick:
\begin{align}\label{gcor+grep}
\nabla_{\theta}\E_{g(\omega|\theta)}[f(\omega)] = \nabla_{\theta} \E_{ \pi(\varepsilon| \theta)} [f(h(\varepsilon, \theta))] = \nonumber
\\
\E_{\pi(\varepsilon| \theta)}[ \nabla_{\theta} f(h(\varepsilon, \theta))] + \nonumber
\\
+ \;\E_{ \pi(\varepsilon| \theta)}\left[f(h(\varepsilon, \theta)) \nabla_{\theta}\log \dfrac{g(h(\varepsilon, \theta)|\theta)}{r(h(\varepsilon,\theta)|\theta)}\right],
\end{align}

However, in the case of the vMF a different procedure is required. After performing the transformation $h(\varepsilon, \theta)$ and accepting/rejecting the sample, we sample \textit{another} random variable $\vv \sim\pi_2(\vv)$, and then apply a transformation $\z = \mathcal{T}(h(\varepsilon, \theta), \vv; \theta)$, such that $\z \sim q_{\psi}(\z| \theta)$ is distributed as the approximate posterior (in our case a vMF). Effectively this entails applying another reparameterization trick after the acceptance/rejection step. To still be able to perform the reparameterization we show that Lemma \ref{normal-repar} fundamentally still holds in this case as well.
\begin{lem}\label{our-repar}
Let $f$ be any measurable function and $\varepsilon \sim \pi_1(\varepsilon| \theta) = s(\varepsilon)\dfrac{g(h(\varepsilon,\theta)| \theta)}{r(h(\varepsilon,\theta)| \theta)}$ the distribution of the accepted sample. 
Also let $\vv \sim\pi_2(v)$, and $\mathcal{T}$ a transformation that depends on the parameters such that if $\z = \mathcal{T}(\omega, v; \theta)$ with $\omega\sim g(\omega|\theta)$, then $ \sim q(\z| \theta) $:
\begin{align}
    \E_{ (\varepsilon,\vv)\sim\pi_1(\varepsilon| \theta)\pi_2(\vv)} \left[f\left(\mathcal{T}(h(\varepsilon, \theta), \vv; \theta)\right)\right] = \nonumber \\
    \int f(\z)q(\z|\theta)d\z = \E_{q(\z|\theta)}[f(\z)],
\end{align}
\begin{proof}
    See Appendix \ref{app:reparameterization-trick}.
\end{proof}
\end{lem}
With this result we are able to derive a gradient expression similarly as done in equation \ref{gcor+grep}. We refer to Appendix \ref{ap:gradient-deriv} for a complete derivation. 

\begin{figure*}[t]
\centering
     \subfigure[$\mathbb{R}^2$ latent space of the $\mathcal{N}$-VAE.]{\includegraphics[width=0.32\textwidth]{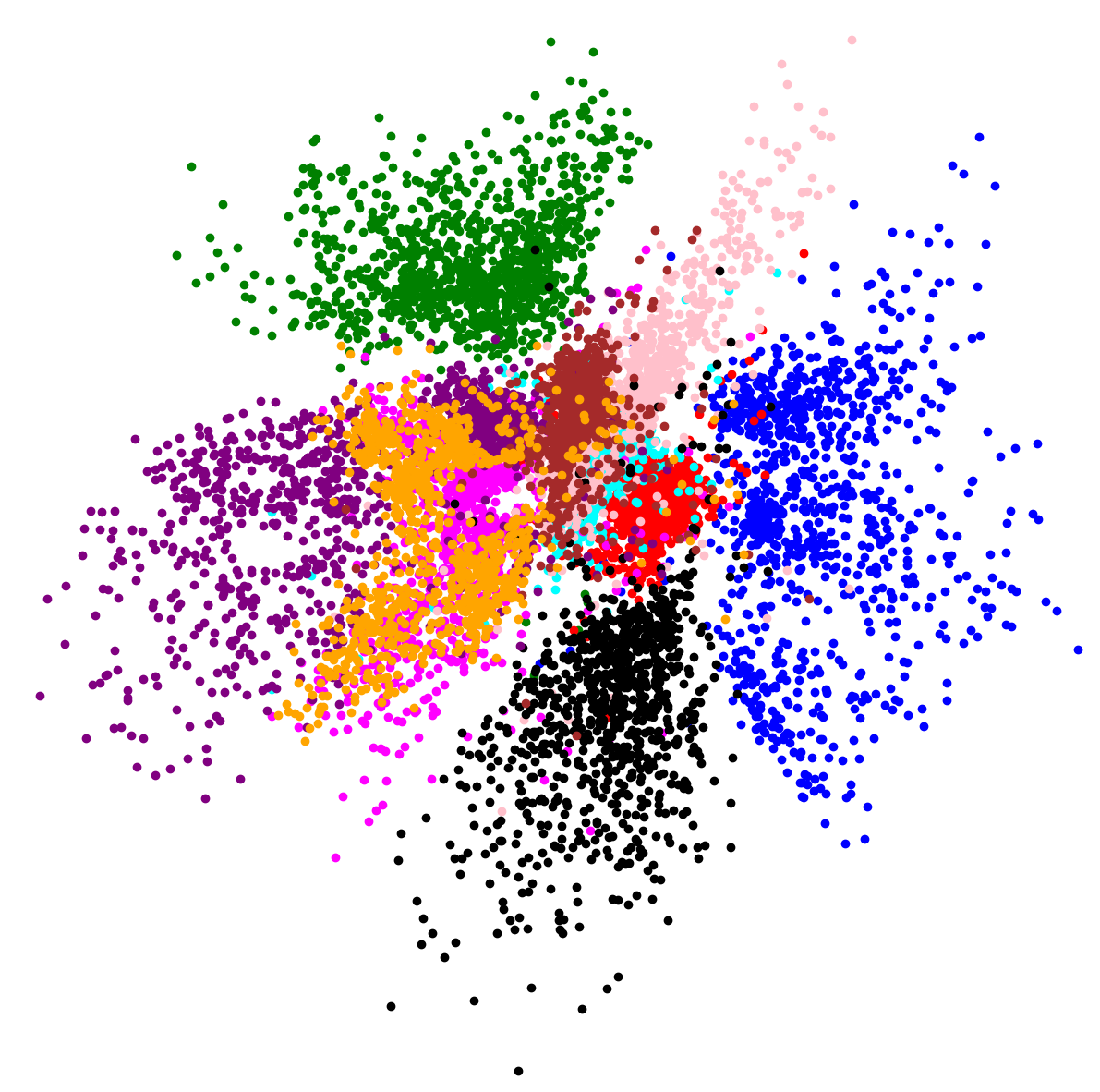} \label{fig:latent_n_mnist}}
     \hspace{5em}
      \subfigure[Hammer projection of $\mathcal{S}^2$ latent space of the $\mathcal{S}$-VAE.]{\includegraphics[width=0.52\textwidth]{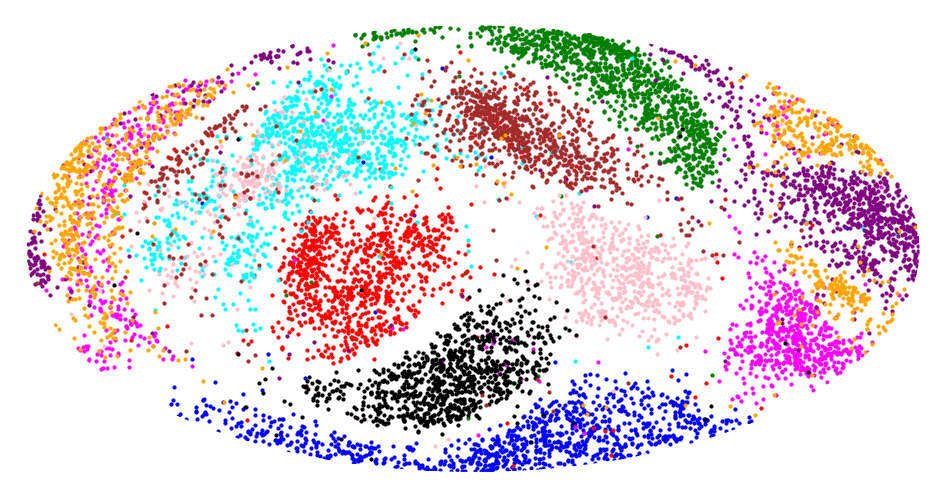}
      \label{fig:latent_s_mnist}}
      \caption{Latent space visualization of the 10 MNIST digits in 2 dimensions of both $\mathcal{N}$-VAE (left) and $\mathcal{S}$-VAE (right). (Best viewed in color)}
      \label{fig:latent_mnist_z2}
\end{figure*}

\subsection{BEHAVIOR IN HIGH DIMENSIONS}\label{subsec:high-dim}
The surface area of a hypersphere is defined as
\begin{equation}
    S(m-1) = r^m\frac{2(\pi^{m/2})}{\Gamma(m/2)}
\end{equation}
where $m$ is the dimensionality and $r$ the radius. Notice that $S(m-1) \to 0$, as $m \to \infty$. However, even for $m > 20$ we observe a \textit{vanishing surface problem} (see Figure \ref{fig:sphere-area} in Appendix \ref{app:surface}). This could thus lead to unstable behavior of hyperspherical models in high dimensions.

\section{RELATED WORK} \label{sec:related-work}

\label{par:vae-ext}
\paragraph{Extending the VAE} The majority of VAE extensions focus on increasing the flexibility of the approximate posterior. This is usually achieved through \textit{normalizing flows} \citep{normalizing-flows}, a class of invertible transformations applied sequentially to an initial reparameterizable density $q_{0}(\z_0)$, allowing for more complex posteriors. Normalizing flows can be considered orthogonal to our proposed approach. In fact, while allowing for a more flexible posterior, they do not modify the standard normal prior assumption. They could be perfectly combined with \Sv-VAEs allowing for more flexible distributions on the hypersphere.

One approach to obtain a more flexible prior is to use a simple mixture of Gaussians (MoG) prior \citep{gmm-prior-vae}. The recently introduced VampPrior model \citep{vamp-prior} outlines several advantages over the MoG and instead tries to learn a more flexible prior by expressing it as a mixture of approximate posteriors. A non-parametric prior is proposed in \citet{stick}, utilizing a truncated stick-breaking process. Opposite to these approaches, we aim at using a non-informative prior to simplify the inference.

The closest approach to ours is a VAE with a vMF distribution in the latent space used for a sentence generation task by \citep{guu2017generating}. While formally this approach is cast as a variational approach, the proposed model does not reparameterize and learn the concentration parameter $\kappa$, treating it as a constant value that remains the same for every approximate posterior instead. Critically, as indicated in Equation \ref{eq:kl}, the KL divergence term only depends on $\kappa$ therefore leaving $\kappa$ constant means never explicitly optimizing the KL divergence term in the loss. The method then only optimizes the reconstruction error by adding vMF noise to the encoder output in the latent space to still allow generation. Moreover, using a fixed global $\kappa$ for \textit{all} the approximate posteriors severely limits the flexibility and the expressiveness of the model.
 
 \begin{table*}[t]
  \centering
  \caption{Summary of results (mean and standard-deviation over 10 runs) of unsupervised model on MNIST. RE and KL correspond respectively to the reconstruction and the KL part of the ELBO. Best results are highlighted only if they passed a student t-test with $p<0.01$.}
  \bigskip
  \begin{tabular}{l|cccc|cccc}
    \toprule
    \multirow{2}{*}{\textbf{Method}} &
        \multicolumn{4}{c}{\textbf{$\mathcal{N}$-VAE}} &
        \multicolumn{4}{c}{\textbf{$\mathcal{S}$-VAE}}
        \\
        & LL & $\mathcal{L}[q]$ & $RE$ & $KL$ & LL & $\mathcal{L}[q]$ & $RE$ & $KL$ \\
    \midrule
$d=2$&-135.73{\tiny$\pm$.83}&-137.08{\tiny$\pm$.83}&-129.84{\tiny$\pm$.91}&7.24{\tiny$\pm$.11}&\textbf{-132.50}{\tiny$\pm$.73}&-133.72{\tiny$\pm$.85}&-126.43{\tiny$\pm$.91}&7.28{\tiny$\pm$.14}\\
$d=5$&-110.21{\tiny$\pm$.21}&-112.98{\tiny$\pm$.21}&-100.16{\tiny$\pm$.22}&12.82{\tiny$\pm$.11}&\textbf{-108.43}{\tiny$\pm$.09}&-111.19{\tiny$\pm$.08}&-97.84{\tiny$\pm$.13}&13.35{\tiny$\pm$.06}\\
$d=10$&-93.84{\tiny$\pm$.30}&-98.36{\tiny$\pm$.30}&-78.93{\tiny$\pm$.30}&19.44{\tiny$\pm$.14}&\textbf{-93.16}{\tiny$\pm$.31}&-97.70{\tiny$\pm$.32}&-77.03{\tiny$\pm$.39}&20.67{\tiny$\pm$.08}\\
$d=20$&-88.90{\tiny$\pm$.26}&-94.79{\tiny$\pm$.19}&-71.29{\tiny$\pm$.45}&23.50{\tiny$\pm$.31}&-89.02{\tiny$\pm$.31}&-96.15{\tiny$\pm$.32}&-67.65{\tiny$\pm$.43}&28.50{\tiny$\pm$.22}\\
$d=40$&\textbf{-88.93}{\tiny$\pm$.30}&-94.91{\tiny$\pm$.18}&-71.14{\tiny$\pm$.56}&23.77{\tiny$\pm$.49}&-90.87{\tiny$\pm$.34}&-101.26{\tiny$\pm$.33}&-67.75{\tiny$\pm$.70}&33.50{\tiny$\pm$.45}\\
    \bottomrule
  \end{tabular}
  \label{tab:mnist}
\end{table*}
 
\label{par:non-euclidean-latent-space}
\paragraph{Non-Euclidean Latent Space}
In \cite{riemann-stein}, a general model to perform Bayesian inference in Riemannian Manifolds is proposed. Following other Stein-related approaches, the method does not explicitly define a posterior density but approximates it with a number of particles. Despite its generality and flexibility, it requires the choice of a kernel on the manifold and multiple particles to have a good approximation of the posterior distribution. The former is not necessarily straightforward, while the latter quickly becomes computationally unfeasible. 

Another approach by \cite{nickel2017poincare}, capitalizes on the hierarchical structure present in some data types. By learning the embeddings for a graph in a non-euclidean negative curvature hyperbolical space, they show this topology has clear advantages over embedding these objects in a Euclidean space. Although they did not use a VAE-based approach, that is, they did not build a probabilistic generative model of the data interpreting the embeddings as latent variables, this approach shows the merit of explicitly adjusting the choice of latent topology to the data used.

\label{par:hyperspherical-perspective}
\paragraph{A Hyperspherical Perspective}
As noted before, a distinction must be made between models dealing with the challenges of intrinsically hyperspherical data like omnidirectional video, and those attempting to exploit some latent hyperspherical manifold. A recent example of the first can be found in \cite{s.2018spherical}, where \textit{spherical} CNNs are introduced. While flattening a spherical image produces unavoidable distortions, the newly defined convolutions take into account its geometrical properties. 

The most general implementation of the second model type was proposed by \cite{gopal2014mises}, who introduced a suite of models to improve cluster performance of high-dimensional data based on mixture of vMF distributions. They showed that reducing an object representation to its directional components increases clusterability over standard methods like $K$-Means or Latent Dirichlet Allocation \citep{blei2003latent}.

Specific applications of the vMF can be further found ranging from computer vision, where it is used to infer structure from motion \citep{guan2017structure} in spherical video, or structure from texture \citep{wilson2014spherical-texture}, to natural language processing, where it is utilized in text analysis \citep{banerjee2003generative-text, banerjee2005clustering} and topic modeling \citep{banerjee2007topic, reisinger:icml10}. 

Additionally, modeling data by restricting it to a hypersphere provides some natural regularizing properties as noted in \citep{liu-nips17-hypersphere-cnn}. Finally \citet{aytekin2018clustering} show on a variety of deep auto-encoder models that adding L2 normalization to the latent space during training, i.e. forcing the latent space on a hypersphere, improves clusterability.

\section{EXPERIMENTS} \label{sec:exp}

In this section, we first perform a series of experiments to investigate the theoretical properties of the proposed \Sv-VAE compared to the \Nv-VAE. In a second experiment, we show how \Sv-VAEs can be used in semi-supervised tasks to create a better separable latent representation to enhance classification. In the last experiment, we show that the \Sv-VAE indeed presents a promising alternative to \Nv-VAEs for data with a non-Euclidean latent representation of low dimensionality, on a link prediction task for three citation networks. All architecture and hyperparameter details are given in Appendix \ref{sect:appDetails}.

\subsection{RECOVERING HYPERSPHERICAL LATENT REPRESENTATIONS} \label{subsec:exp-recover-latent}

In this first experiment we build on the motivation developed in Subsection \ref{subsec:beyond-hyperplanes}, by confirming with a synthetic data example the difference in behavior of the \Nv-VAE and \Sv-VAE in recovering latent hyperspheres. We first generate samples from a mixture of three vMFs on the circle, $\mathcal{S}^1$, as shown in Figure \ref{fig:toy1_orig}, which subsequently are mapped into the higher dimensional $\mathbb{R}^{100}$ by applying a noisy, non-linear transformation. After this, we in turn train an auto-encoder, a \Nv-VAE, and a \Sv-VAE. We further investigate the behavior of the \Nv-VAE, by training a model using a scaled down KL divergence.

\begin{table*}[!ht]
  \centering
  \caption{Summary of results (mean accuracy and standard-deviation over 20 runs) of semi-supervised $K$-NN on MNIST. Best results are highlighted only if they passed a student t-test with $p<0.01$.}
  \bigskip
  \begin{tabular}{l|cc|cc|cc}
    \toprule
    \multirow{2}{*}{\textbf{Method}} &
    \multicolumn{2}{c}{\textbf{100}} &
    \multicolumn{2}{c}{\textbf{600}} &
    \multicolumn{2}{c}{\textbf{1000}}
        \\
        & $\mathcal{N}$-VAE & $\mathcal{S}$-VAE & $\mathcal{N}$-VAE & $\mathcal{S}$-VAE & $\mathcal{N}$-VAE & $\mathcal{S}$-VAE \\
    \midrule
$d=2$&72.6{\tiny$\pm$2.1}&\textbf{77.9}{\tiny$\pm$1.6}&80.8{\tiny$\pm$0.5}&\textbf{84.9}{\tiny$\pm$0.6}&81.7{\tiny$\pm$0.5}&\textbf{85.6}{\tiny$\pm$0.5}\\
$d=5$&81.8{\tiny$\pm$2.0}&\textbf{87.5}{\tiny$\pm$1.0}&90.9{\tiny$\pm$0.4}&\textbf{92.8}{\tiny$\pm$0.3}&92.0{\tiny$\pm$0.2}&\textbf{93.4}{\tiny$\pm$0.2}\\
$d=10$&75.7{\tiny$\pm$1.8}&\textbf{80.6}{\tiny$\pm$1.3}&88.4{\tiny$\pm$0.5}&\textbf{91.2}{\tiny$\pm$0.4}&90.2{\tiny$\pm$0.4}&\textbf{92.8}{\tiny$\pm$0.3}\\
$d=20$&71.3{\tiny$\pm$1.9}&\textbf{72.8}{\tiny$\pm$1.6}&88.3{\tiny$\pm$0.5}&\textbf{89.1}{\tiny$\pm$0.6}&90.1{\tiny$\pm$0.4}&\textbf{91.1}{\tiny$\pm$0.3}\\
$d=40$&\textbf{72.3}{\tiny$\pm$1.6}&67.7{\tiny$\pm$2.3}&88.0{\tiny$\pm$0.5}&87.4{\tiny$\pm$0.7}&90.3{\tiny$\pm$0.5}&90.4{\tiny$\pm$0.4}\\
    \bottomrule
    \end{tabular}
  \label{tab:m1_k5_a}
\end{table*}

\paragraph{Results}
The resulting latent spaces, displayed in Figure \ref{fig:exp-toy1}, clearly confirm the intuition built in Subsection \ref{subsec:beyond-hyperplanes}. As expected, in Figure \ref{fig:toy1_AE} the auto-encoder is perfectly capable to embed in low dimensions the original underlying data structure. However, most parts of the latent space are not occupied by points, critically affecting the ability to generate meaningful samples.

In the \Nv-VAE setting we observe two types of behaviours, summarized by Figures \ref{fig:toy1_VAE} and  \ref{fig:toy1_VAE2}. In the first we observe that if the prior is too strong it will force the posterior to match the prior shape, concentrating the samples in the center. However, this prevents the \Nv-VAE to correctly represent the true shape of the data and creates instability problems for the decoder around the origin. On the contrary, if we scale down the KL term, we observe that the samples from the approximate posterior maintain a shape that reflects the $\mathcal{S}^1$ structure smoothed with Gaussian noise. However, as the approximate posterior differs strongly from the prior, obtaining meaningful samples from the latent space again becomes problematic. 

The \Sv-VAE on the other hand, almost perfectly recovers the original dataset structure, while the samples from the approximate posterior closely match the prior distribution. This simple experiment confirms the intuition that having a prior that matches the true latent structure of the data, is crucial in constructing a correct latent representation that preserves the ability to generate meaningful samples. 

\subsection{EVALUATION OF EXPRESSIVENESS} \label{subsec:mnist}
To compare the behavior of the $\mathcal{N}$-VAE and $\mathcal{S}$-VAE on a data set that does not have a clear hyperspherical latent structure, we evaluate both models on a reconstruction task using dynamically binarized MNIST \citep{salakhutdinov2008quantitative}. We analyze the ELBO, KL, negative reconstruction error, and marginal log-likelihood (LL) for both models on the test set. The LL is estimated using importance sampling with 500 sample points \citep{burda2015importance}.

\paragraph{Results}
Results are shown in Table \ref{tab:mnist}. We first note that in terms of negative reconstruction error the \Sv-VAE outperforms the \Nv-VAE in all dimensions. Since the \Sv-VAE uses a uniform prior, the KL divergence increases more strongly with dimensionality, which results in a higher ELBO. However in terms of log-likelihood (LL) the \Sv-VAE clearly has an edge in low dimensions ($d = 2,5,10$) and performs comparable to the \Nv-VAE in $d=20$. This empirically confirms the hypothesis of Subsection \ref{subsec:lim-gaussian-prior}, showing the positive effect of having a uniform prior in low dimensions. In the absence of any origin pull, the data is able to cluster naturally, utilizing the entire latent space which can be observed in Figure \ref{fig:latent_mnist_z2}. Note that in Figure \ref{fig:latent_n_mnist} all mass is concentrated around the center, since the prior mean is zero. Conversely, in Figure \ref{fig:latent_s_mnist} all available space is evenly covered due to the uniform prior, resulting in more separable clusters in $\mathcal{S}^2$ compared to $\mathbb{R}^2$. However, as dimensionality increases, the Gaussian distribution starts to approximate a hypersphere, while its posterior becomes more expressive than the vMF due to the higher number of variance parameters. Simultaneously, as described in Subsection \ref{subsec:high-dim}, the surface area of the vMF starts to collapse limiting the available space.

In Figure \ref{fig:nvae_samples} and \ref{fig:svae_samples} of Appendix \ref{app:mnist-latent-visual}, we present randomly generated samples from the \Nv-VAE and the \Sv-VAE, respectively. Moreover, in Figure \ref{fig:latent_spaces_appendix} of Appendix \ref{app:mnist-latent-visual}, we show 2-dimensional manifolds for the two models. Interestingly, the manifold given by the \Sv-VAE indeed results in a latent space where digits occupy the entire space and there is a sense of continuity from left to right.

\begin{figure*}[h!]
\centering
     \subfigure[$\mathbb{R}^2$ latent space of the \Nv-VGAE.]{\includegraphics[width=0.32\textwidth]{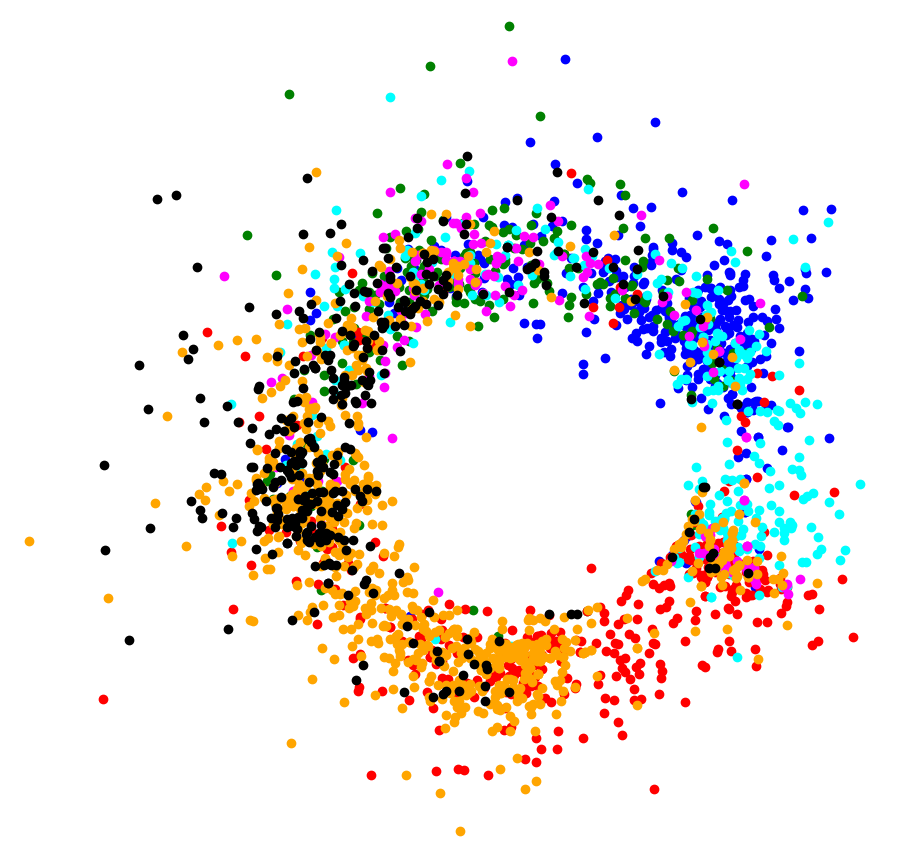} \label{fig:latent_n_cora}}
     \hspace{5em}
      \subfigure[Hammer projection of $\mathcal{S}^2$ latent space of the \Sv-VGAE.]{\includegraphics[width=0.52\textwidth]{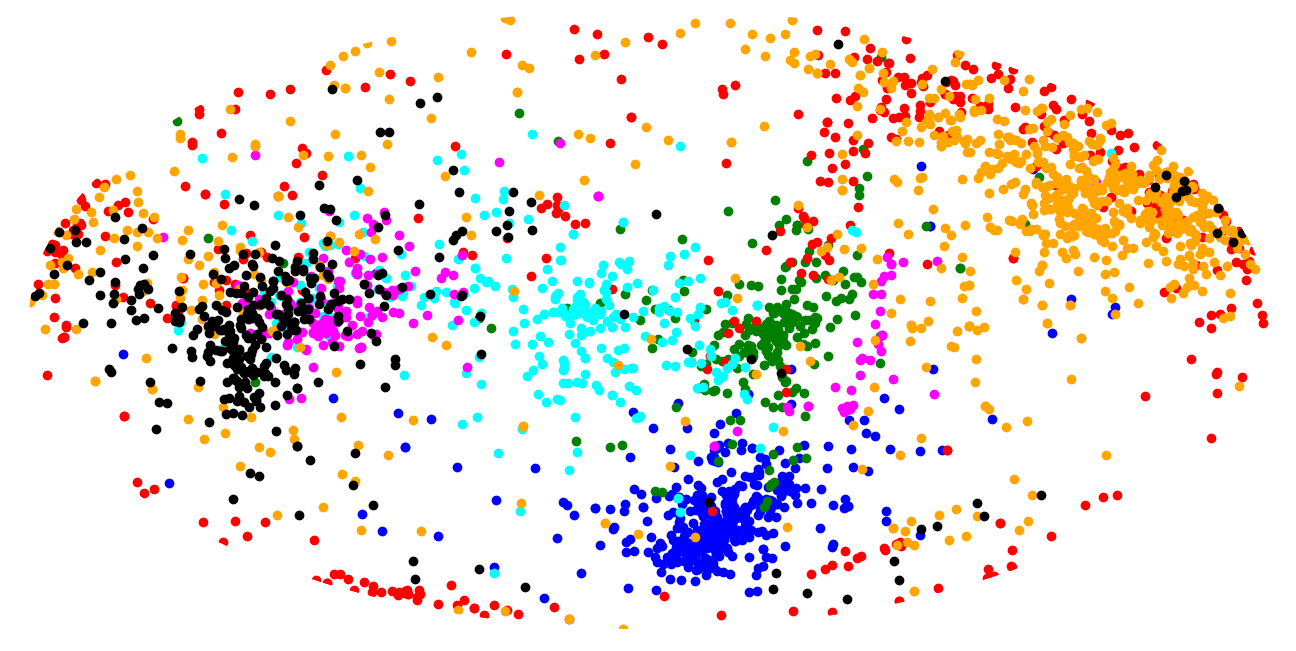}\label{fig:latent_s_cora}}
      \caption{Latent space of unsupervised \Nv-VGAE and \Sv-VGAE models trained on Cora citation network. Colors denote documents classes which are not provided during training. (Best viewed in color)}
\end{figure*}

\subsection{SEMI-SUPERVISED LEARNING} \label{subsec:m1}

Having observed the \Sv-VAE's ability to increase clusterability of data points in the latent space, we wish to further investigate this property using a semi-supervised classification task. For this purpose we re-implemented the M1 and M1+M2 models as described in \citep{kingma-semi-super}, and evaluate the classification accuracy of the \Sv-VAE and the \Nv-VAE on dynamically binarized MNIST. In the M1 model, a classifier utilizes the latent features obtained using a VAE as in experiment \ref{subsec:mnist}. The M1+M2 model is constructed by stacking the M2 model on top of M1, where M2 is the result of augmenting the VAE by introducing a partially observed variable $\y$, and combining the ELBO and classification objective. This concatenated model is trained end-to-end \footnote{It is worth noting that in the original implementation by \citet{kingma-semi-super} the stacked model did not converge well using end-to-end training, and used the extracted features of the M1 model as inputs for the M2 model instead.}. 

This last model also allows for a combination of the two topologies due to the presence of two distinct latent variables, $\z_1$ and $\z_2$. Since in the M2 latent space the class assignment is expressed by the variable $\y$, while $\z_2$ only needs to capture the style, it naturally follows that the \Nv-VAE is more suited for this objective due to its higher number of variance parameters. Hence, besides comparing the \Sv-VAE against the \Nv-VAE, we additionally run experiments for the M1+M2 model by modeling $\z_1$, $\z_2$ respectively with a vMF and normal distribution.

\paragraph{Results}

As can be see in Table \ref{tab:m1_k5_a}, for M1 the \Sv-VAE outperforms the \Nv-VAE in all dimensions up to $d=40$. This result is amplified for a low number of observed labels. Note that for both models absolute performance drops as the dimensionality increases, since $K$-NN used as the classifier suffers from the curse of dimensionality. Besides reconfirming superiority of the \Sv-VAE in $d < 20$, its better performance than the \Nv-VAE for $d=20$ was unexpected. This indicates that although the log-likelihood might be comparable(see Table \ref{tab:mnist}) for higher dimensions, the \Sv-VAE latent space better captures the cluster structure.

In the concatenated model M1+M2, we first observe in Table \ref{tab:m1m2} that either the pure \Sv-VAE or the \Sv+\Nv-VAE model yields the best results, where the \Sv-VAE almost always outperforms the \Nv-VAE. Our hypothesis regarding the merit of a \Sv+\Nv-VAE model is further confirmed, as displayed by the stable, strong performance across all different dimensions. Furthermore, the clear edge in clusterability of the \Sv-VAE in low dimensional $\z_1$ as already observed in Table \ref{tab:m1_k5_a}, is again evident. As the dimensionality of $\z_1, \z_2$ increases, the accuracy of the \Nv-VAE improves, reducing the performance gap with the \Sv-VAE. As previously noticed the \Sv-VAE performance drops when $dim_{\ \z_2} = 50$, with the best result being obtained for $dim_{\ \z_1} = dim_{\ \z_2} = 10$. In fact, it is worth noting that for this setting the \Sv-VAE obtains comparable results to the original settings of \citep{kingma-semi-super}, while needing a considerably smaller latent space. Finally, the end-to-end trained \Sv+\Nv-VAE model is able to reach a significantly higher classification accuracy than the original results reported by \citet{kingma-semi-super}, 96.7{\tiny$\pm.1$}.

The M1+M2 model allows for conditional generation. Similarly to \citep{kingma-semi-super}, we set the latent variable $\z_2$ to the value inferred from the test image by the inference network, and then varied the class label $\y$. In Figure \ref{fig:latent_spaces} of Appendix \ref{app:cond-visual} we notice that the model is able to disentangle the style from the class.

\begin{table}[!ht]
    \centering
    \caption{Summary of results of semi-supervised model M1+M2 on MNIST.}
    \bigskip
    \begin{tabular}{cc|ccc}
    \toprule
    \multicolumn{2}{c}{\textbf{Method}} &
      \multicolumn{3}{c}{\textbf{100}} \\
      $dim_{\ \z_1}$ & $dim_{\ \z_2}$
      & \Nv+\Nv & \Sv+\Sv & \Sv+\Nv  \\
        \midrule
    \multirow{3}{*}{5}  & 5   & 90.0{\tiny$\pm$.4} & {\bf94.0}{\tiny$\pm$.1} & 93.8{\tiny$\pm$.1} \\
                        & 10  & 90.7{\tiny$\pm$.3} & 94.1{\tiny$\pm$.1} & {\bf94.8}{\tiny$\pm$.2} \\
                        & 50  & 90.7{\tiny$\pm$.1} & 92.7{\tiny$\pm$.2} & {\bf93.0}{\tiny$\pm$.1} \\
    \midrule
    \multirow{3}{*}{10} & 5   & 90.7{\tiny$\pm$.3} & 91.7{\tiny$\pm$.5} & {\bf94.0}{\tiny$\pm$.4} \\
                        & 10  & 92.2{\tiny$\pm$.1} & {\bf96.0}{\tiny$\pm$.2} & {\bf95.9}{\tiny$\pm$.3}\\
                        & 50  & 92.9{\tiny$\pm$.4} & 95.1{\tiny$\pm$.2} & {\bf95.7}{\tiny$\pm$.1} \\
    \midrule
    \multirow{3}{*}{50} & 5   & 92.0{\tiny$\pm$.2} & 91.7{\tiny$\pm$.4} & {\bf95.8}{\tiny$\pm$.1} \\
                        & 10  & 93.0{\tiny$\pm$.1} & 95.8{\tiny$\pm$.1} & {\bf97.1}{\tiny$\pm$.1} \\
                        & 50  & 93.2{\tiny$\pm$.2} & 94.2{\tiny$\pm$.1} & {\bf97.4}{\tiny$\pm$.1} \\
    \bottomrule
    \end{tabular}
    \label{tab:m1m2}
\end{table}

\subsection{LINK PREDICTION ON GRAPHS}

In this experiment, we aim at demonstrating the ability of the \Sv-VAE to learn meaningful embeddings of nodes in a graph, showing the advantages of embedding objects in a non-Euclidean space. We test hyperspherical reparameterization on the recently introduced Variational Graph Auto-Encoder (VGAE) \citep{kipf2016VGAE}, a VAE model for graph-structured data. We perform training on a link prediction task on three popular citation network datasets \citep{sen2008collective}: Cora, Citeseer and Pubmed.

Dataset statistics and further experimental details are summarized in Appendix \ref{ap:link-pred}. The models are trained in an unsupervised fashion on a masked version of these datasets where some of the links have been removed. All node features are provided and efficacy is measured in terms of average precision (AP) and area under the ROC curve (AUC) on a test set of previously removed links. We use the same training, validation, and test splits as in \cite{kipf2016VGAE}, i.e. we assign 5\% of links for validation and 10\% of links for testing.

\begin{table}[H]
  \centering
    \caption{Results for link prediction in citation networks.}
    \bigskip
    \begin{tabular}{ll|cccc}
    \toprule
      \textbf{Method}  &  & \textbf{\Nv-VGAE} & \textbf{\Sv-VGAE} \\
    \midrule
        \multirow{2}{*}{\textbf{Cora}} & AUC & 92.7{\tiny$\pm$.2} & \textbf{94.1}{\tiny$\pm$.1} \\
        & AP &  93.2{\tiny$\pm$.4} & \textbf{94.1}{\tiny$\pm$.3} \\
    \midrule
        \multirow{2}{*}{\textbf{Citeseer}}& AUC & 90.3{\tiny$\pm$.5} & \textbf{94.7}{\tiny$\pm$.2} \\
        & AP & 91.5{\tiny$\pm$.5} & \textbf{95.2}{\tiny$\pm$.2} \\
    \midrule
    \multirow{2}{*}{\textbf{Pubmed}}& AUC & \textbf{97.1}{\tiny$\pm$.0} & 96.0{\tiny$\pm$.1} \\
    & AP & \textbf{97.1}{\tiny$\pm$.0} & 96.0{\tiny$\pm$.1} \\
    \bottomrule
    \end{tabular}
  \label{tab:graph}
\end{table}

\paragraph{Results} In Table \ref{tab:graph}, we show that our model outperforms the \Nv-VGAE baseline on two out of the three datasets by a significant margin. The log-probability of a link is computed as the dot product of two embeddings. In a hypersphere, this can be interpreted as the cosine similarity between vectors. Indeed we find that the choice of a dot product scoring function for link prediction is problematic in combination with the normal distribution on the latent space. If embeddings are close to the zero-center, noise during training can have a large destabilizing effect on the angle information between two embeddings. In practice, the model finds a solution where embeddings are "pushed" away from the zero-center, as demonstrated in Figure \ref{fig:latent_n_cora}. This counteracts the pull towards the center arising from the standard prior and can overall lead to poor modeling performance. By constraining the embeddings to the surface of a hypersphere, this effect is mitigated, and the model can find a good separation of the latent clusters, as shown in Figure \ref{fig:latent_s_cora}. 

On Pubmed, we observe that the \Sv-VAE converges to a lower score than the \Nv-VAE. The Pubmed dataset is significantly larger than Cora and Citeseer, and hence more complex. The \Nv-VAE has a larger number of variance parameters for the posterior distribution, which might have played an important role in better modeling the relationships between nodes. We further hypothesize that not all graphs are necessarily better embedded in a hyperspherical space and that this depends on some fundamental topological properties of the graph. For instance, the already mentioned work from \citet{nickel2017poincare} shows that hyperbolical space is better suited for graphs with a hierarchical, tree-like structure. These considerations prefigure an interesting research direction that will be explored in future work.

\section{CONCLUSION}

With the \Sv-VAE we set an important first step in the exploration of hyperspherical latent representations for variational auto-encoders. Through various experiments, we have shown that \Sv-VAEs have a clear advantage over \Nv-VAEs for data residing on a known hyperspherical manifold, and are competitive or surpass \Nv-VAEs for data with a non-obvious hyperspherical latent representation in lower dimensions. Specifically, we demonstrated \Sv-VAEs improve separability in semi-supervised classification and that they are able to improve results on state-of-the-art link prediction models on citation graphs, by merely changing the prior and posterior distributions as a simple drop-in replacement.

We believe that the presented research paves the way for various promising areas of future work, such as exploring more flexible approximate posterior distributions through normalizing flows on the hypersphere, or hierarchical mixture models combining hyperspherical and hyperplanar space. Further research should be done in increasing the performance of \Sv-VAEs in higher dimensions; one possible solution of which could be to dynamically learn the radius of the latent hypersphere in a full Bayesian setting.

\section*{Acknowledgements}
We would like to thank Rianne van den Berg, Jonas K\"ohler, Pim de Haan, Taco Cohen, Marco Federici, and Max Welling for insightful discussions. T.K.~is supported by the SAP Innovation Center Network. J.M.T.~was funded by the European Commission within the Marie Sk\l odowska-Curie Individual Fellowship (Grant No. 702666, ”Deep learning and Bayesian inference for medical imaging”).

\bibliography{main}
\bibliographystyle{apalike} 

\onecolumn

\appendix
\section{SAMPLING PROCEDURE} \label{app:vmf-sample}

\begin{figure}[H]\label{fig:sampling}
  \centering
  \includegraphics[width=0.85\textwidth]{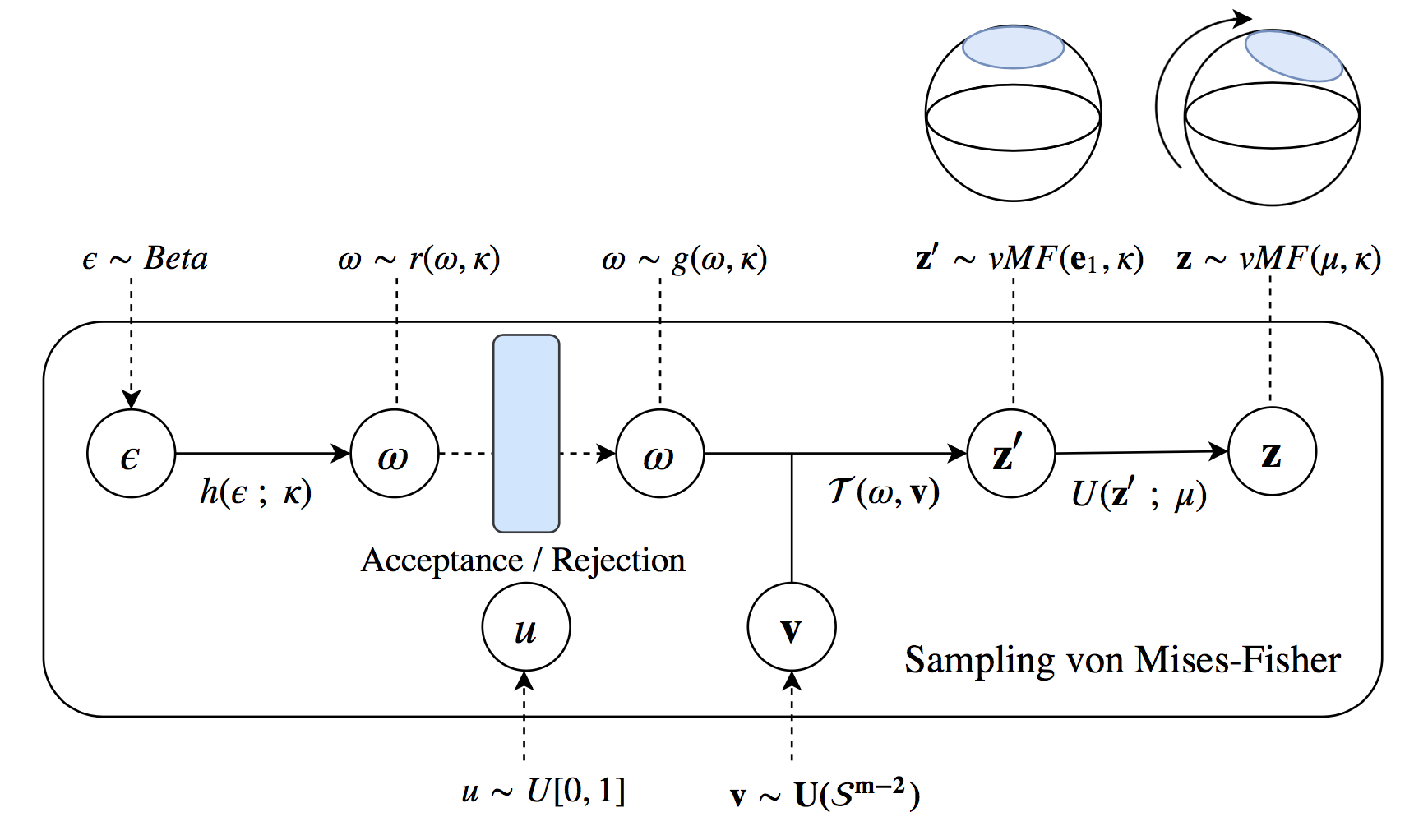}
  \caption{Overview of von Mises-Fisher sampling procedure. Note that as $\omega$ is a scalar, the procedure does not suffer from the curse of dimensionality.}
\end{figure}
 
The general algorithm for sampling from a vMF has been outlined in Algorithm \ref{alg:vmf-sample}. 
The exact form of the distribution of the univariate distribution $g(\omega| k)$ is:
\begin{align}
    g(\omega| k) = \dfrac{2(\pi^{m/2})}{\Gamma(m/2)} \mathcal{C}_m(k) \dfrac{\exp(\omega k)(1-\omega^2)^{\frac{1}{2}(m-3)}}{B(\frac{1}{2}, \frac{1}{2}(m-1))},
\end{align}
Samples from this distribution are drawn using an acceptance/rejection algorithm when $m \neq 3$. The complete procedure is described in Algorithm \ref{alg:g-sample}. The $Householder$ reflection (see Algorithm \ref{alg:householder-transform} for details) simply finds an orthonormal transformation that maps the modal vector $\e_1 = (1,0,\cdots,0)$ to $\mu$. Since an orthonormal transformation preserves the distances all the points in the hypersphere will stay in the surface after mapping. Notice that even the transform $U\z' =  (\mathbb{I} - 2\uv  \uv^{\top})\z'$, can be executed in $\mathcal{O}(m)$ by rearranging the terms. 

\begin{figure}[H]\label{fig:sphere-sample-example}
  \centering
  \includegraphics[width=0.49\textwidth]{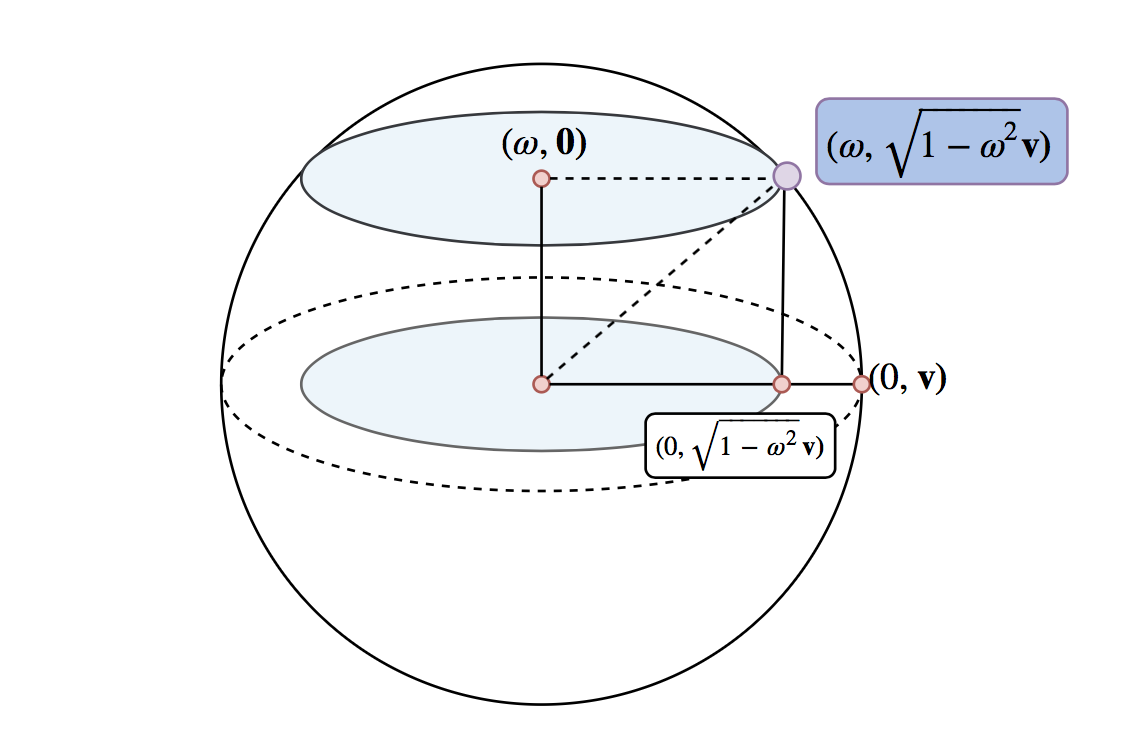}
  \caption{Geometric representation of a single sample in $\mathcal{S}^2$, where $\omega \sim g(\omega|k)$ and $\vv \sim U(\mathcal{S}^{1})$.}
\end{figure}

\begin{minipage}[t]{8cm}
  \vspace{0pt}  
\begin{algorithm}[H]
    \centering
   \caption{$g(\omega| k)$ acceptance-rejection sampling}
   \label{alg:g-sample}
\begin{algorithmic}
    \STATE {\bfseries Input:} dimension $m$, concentration $\kappa$
    \STATE Initialize values:
    \STATE $b \gets \dfrac{-2k + \sqrt{4k^2 + (m-1)^2}}{m - 1}$\\
    \STATE $a \gets \dfrac{(m-1) + 2k +\sqrt{4k^2+(m-1)^2}}{4}$\\
    \STATE $d \gets \dfrac{4ab}{(1+b)} - (m - 1)\ln (m-1)$
    \REPEAT
    \STATE Sample $\varepsilon \sim \text{Beta}(\frac{1}{2}(m-1), \frac{1}{2}(m-1))$
    \STATE $\omega \gets h(\varepsilon, k) = \dfrac{1 - (1 + b)\varepsilon}{1 - (1 - b)\varepsilon}$
    \STATE $t \gets \dfrac{2ab}{1 - (1-b)\varepsilon}$
    \STATE Sample $u \sim \mathcal{U}(0,1)$
    \UNTIL{$\quad (m-1)\ln(t) - t + d \ge \ln(u)$}
    \STATE{\bfseries Return:}{$\ \omega$}
\end{algorithmic}
\end{algorithm}
\end{minipage}%
\hspace{2 cm}
\begin{minipage}[t]{6cm}
  \vspace{0pt}
\begin{algorithm}[H]
   \caption{Householder transform}
   \label{alg:householder-transform}
\begin{algorithmic}
    \STATE {\bfseries Input:} mean $\mu$, modal vector $\e_1$
    \STATE $\uv' \gets \e_1 - \mu$
    \STATE $\uv \gets \frac{\uv'}{||\uv'||}$
    \STATE $U \gets \mathbb{I} - 2\uv \uv^{\top}$
    \STATE{\bfseries Return:}{$\ U$}
\end{algorithmic}
\end{algorithm}
\end{minipage}

\begin{table}[H]
  \centering
    \caption{Expected number of samples needed before acceptance, computed using Monte Carlo estimate with 1000 samples varying dimensionality and concentration parameters. Notice that the sampling complexity increases in $\kappa$, but decreases as the dimensionality, $d$, increases.}
    \bigskip
    \begin{tabular}{l|ccccccccc}
    \toprule
    & $\kappa=1$ & $\kappa=5$ & $\kappa=10$ & $\kappa=50$ & $\kappa=100$ & $\kappa=500$ & $\kappa=1000$ & $\kappa=5000$ & $\kappa=10000$ \\
    \midrule
$d=5$ & 1.020 & 1.171 & 1.268 & 1.398 & 1.397 & 1.426 & 1.458 & 1.416 & 1.440\\
$d=10$ & 1.008 & 1.094 & 1.154 & 1.352 & 1.411 & 1.407 & 1.369 & 1.402 & 1.419\\
$d=20$ & 1.001 & 1.031 & 1.085 & 1.305 & 1.342 & 1.367 & 1.409 & 1.410 & 1.407\\
$d=40$ & 1.000 & 1.011 & 1.027 & 1.187 & 1.288 & 1.397 & 1.433 & 1.402 & 1.423\\
$d=100$ & 1.000 & 1.000 & 1.006 & 1.092 & 1.163 & 1.317 & 1.360 & 1.398 & 1.416\\
    \bottomrule
    \end{tabular}
\end{table}

\section{KL DIVERGENCE DERIVATION}\label{ap:kl-divergence}

The KL divergence between a von-Mises-Fisher distribution $q(\z| \mu, k)$ and an uniform distribution in the hypersphere (one divided by the surface area of $\mathcal{S}^{m-1}$) $p(\x) = \left( \dfrac{2(\pi^{m/2})}{\Gamma(m/2)} \right)^{-1}$ is:
\begin{align}
    \mathcal{KL}[q(\z|\mu, k) \; || \;p(\z)] &= \int_{\mathcal{S}^{m-1}} q(\z|\mu, k) \log \frac{q(\z|\mu, k)}{p(\z)} d\z \\
    &= \int_{\mathcal{S}^{m-1}} q(\z|\mu, k) \left( \log \mathcal{C}_m (k) + k \mathbf{\mu}^T \z - \log p(\z) \right) d\z 
    \\
    &= k \mu\; \mathbb{E}_q[\z] + \log \mathcal{C}_m(k) - \log \left( \frac{2(\pi^{m/2})}{\Gamma(m/2)} \right)^{-1} \\
    &= k \frac{\mathcal{I}_{m/2}(k)}{\mathcal{I}_{m/2 - 1}(k)} + \left( (m/2 - 1) \log k - (m/2) \log(2\pi) - \log \mathcal{I}_{m/2 - 1}(k)  \right) \\ &\quad + \frac{m}{2} \log \pi + \log 2 - \log  \Gamma(\frac{m}{2}), \nonumber
\end{align}

\subsection{GRADIENT OF KL DIVERGENCE}
Using
\begin{align}
\nabla_k \mathcal{I}_{v}(k) = \frac12 \left(\mathcal{I}_{v-1}(k) + \mathcal{I}_{v+1}(k) \right),
\end{align}
and
\begin{align}
\nabla_k \log \mathcal{C}_m(k) &=
\nabla_k \left( (m/2 - 1) \log k - (m/2) \log(2\pi) - \log \mathcal{I}_{m/2 - 1}(k)  \right) \\
&= - \frac{\mathcal{I}_{m/2}(k)}{\mathcal{I}_{m/2 - 1}(k)},
\end{align}
then
\begin{align}
    \nabla_\kappa \mathcal{KL}[q(\z|\mu, k) \; || \;p(\z)] &= \nabla_\kappa k \frac{\mathcal{I}_{m/2}(k)}{\mathcal{I}_{m/2 - 1}(k)} + \nabla_k \log \mathcal{C}_m(k)   \\
    & = \frac{\mathcal{I}_{m/2}(k)}{\mathcal{I}_{m/2 - 1}(k)} + k \nabla_k  \frac{\mathcal{I}_{m/2}(k)}{\mathcal{I}_{m/2 - 1}(k)} - \frac{\mathcal{I}_{m/2}(k)}{\mathcal{I}_{m/2 - 1}(k)}  \\
    &= \frac12 k \left( \frac{\mathcal{I}_{m/2+1}(k)}{\mathcal{I}_{m/2 - 1}(k)} - \frac{\mathcal{I}_{m/2}(k) \left( \mathcal{I}_{m/2 - 2}(k) + \mathcal{I}_{m/2}(k) \right) }{\mathcal{I}_{m/2 - 1}(k)^2} + 1 \right),
\end{align}
Notice that we can use $\mathcal{I}_{m/2}^{exp} = \exp(-k)\mathcal{I}_{m/2}$ for numerical stability.

\section{PROOF OF LEMMA \ref{our-repar} } \label{app:reparameterization-trick}
\begin{lem}[2]
Let $f$ be any measurable function and $\varepsilon \sim \pi_1(\varepsilon| \theta) = s(\varepsilon)\dfrac{g(h(\varepsilon,\theta)| \theta)}{r(h(\varepsilon,\theta)| \theta)}$ the distribution of the accepted sample. 
Also let $\vv \sim\pi_2(\vv)$, and $\mathcal{T}$ a transformation that depends on the parameters such that if $\z = \mathcal{T}(\omega, \vv; \theta)$ with $\omega\sim g(\omega|\theta)$, then $ \z \sim q(\z| \theta) $:
\begin{align}
    \E_{ (\varepsilon,\vv)\sim\pi_1(\varepsilon| \theta)\pi_2(\vv)} \left[f\left(\mathcal{T}(h(\varepsilon, \theta), \vv; \theta)\right)\right] = \int f(\z)q(\z|\theta)d\z = \E_{q(\z|\theta)}[f(\z)],
\end{align}
\end{lem}
\begin{proof}
\begin{align}
    \E_{ (\varepsilon,\vv)\sim\pi_1(\varepsilon| \theta)\pi_2(\vv)} \left[f\left(\mathcal{T}(h(\varepsilon, \theta), \vv; \theta)\right)\right] = \int\!\!\!\!\int f\left(\mathcal{T}(h(\varepsilon, \theta), \vv; \theta)\right) \pi_1(\varepsilon| \theta)\pi_2(\vv) d\varepsilon d\vv,
\end{align}
    
Using the same argument employed by \citet{rejection-repar} we can apply the change of variables $\omega = h(\varepsilon, \theta)$ rewrite the expression as:
\begin{equation}
     = \int\!\!\!\!\int f\left(\mathcal{T}(\omega, \vv; \theta)\right) g(\omega| \theta)\pi_2(\vv) d\omega d\vv =^* \int f(\z)q(\z|\theta)d\z 
\end{equation}
Where in * we applied the change of variables $\z = \mathcal{T}(\omega, \vv; \theta)$.
\end{proof}

\section{REPARAMETRIZATION GRADIENT DERIVATION} \label{ap:gradient-deriv}
\subsection{GENERAL EXPRESSION DERIVATION}
We can then proceed as in \ref{gcor+grep} using Lemma \ref{our-repar} and the the log derivative trick to compute the gradient of the expectation term $\nabla_{\theta}\E_{q(\z|\theta)}[f(\z)]$:
\begin{align}
    \nabla_{\theta}\E_{q(\z|\theta)}[f(\z)] &= \nabla_{\theta} \int\!\!\!\!\int f\left(\mathcal{T}(h(\varepsilon, \theta), \vv; \theta)\right) \pi_1(\varepsilon| \theta)\pi_2(\vv) d\varepsilon d\vv \\ 
    &=  \nabla_{\theta}\int\!\!\!\!\int f\left(\mathcal{T}(h(\varepsilon, \theta), \vv; \theta)\right) s(\varepsilon)\frac{g(h(\varepsilon,\theta)| \theta)}{r(h(\varepsilon,\theta)| \theta)}\pi_2(\vv) d\varepsilon d\vv \\ 
    &= \int\!\!\!\!\int s(\varepsilon)\pi_2(\vv) \nabla_{\theta}\left(f\left(\mathcal{T}(h(\varepsilon, \theta), \vv; \theta)\right) \frac{g(h(\varepsilon,\theta)| \theta)}{r(h(\varepsilon,\theta)| \theta)}\right) d\varepsilon d\vv \\
    &= \int\!\!\!\!\int s(\varepsilon)\pi_2(\vv) \frac{g(h(\varepsilon,\theta)| \theta)}{r(h(\varepsilon,\theta)| \theta)}\nabla_{\theta}\left(f\left(\mathcal{T}(h(\varepsilon, \theta), \vv; \theta)\right) \right) d\varepsilon d\vv \\
    & \quad +\int\!\!\!\!\int s(\varepsilon)\pi_2(\vv) f\left(\mathcal{T}(h(\varepsilon, \theta), \vv; \theta)\right)\nabla_{\theta}\left( \frac{g(h(\varepsilon,\theta)| \theta)}{r(h(\varepsilon,\theta)| \theta)}\right) d\varepsilon d\vv \nonumber  \\ 
    &= \int\!\!\!\!\int \pi_1(\varepsilon| \theta)\pi_2(\vv) \nabla_{\theta}\left(f\left(\mathcal{T}(h(\varepsilon, \theta), \vv; \theta)\right) \right) d\varepsilon dv 
    \\ 
    &\quad +\int\!\!\!\!\int s(\varepsilon)\pi_2(\vv) f\left(\mathcal{T}(h(\varepsilon, \theta), \vv; \theta)\right)\nabla_{\theta}\left( \frac{g(h(\varepsilon,\theta)| \theta)}{r(h(\varepsilon,\theta)| \theta)}\right) d\varepsilon d\vv \nonumber
    \\ 
    &= \underbrace{\E_{ (\varepsilon,\vv)\sim\pi_1(\varepsilon| \theta)\pi_2(\vv)} \left[\nabla_{\theta} f\left(\mathcal{T}(h(\varepsilon, \theta), \vv; \theta)\right)\right]}_{g_{rep}} \\
    & \quad + \underbrace{\E_{ (\varepsilon,\vv)\sim\pi_1(\varepsilon| \theta)\pi_2(\vv)} \left[ f\left(\mathcal{T}(h(\varepsilon, \theta), \vv; \theta)\right) \nabla_{\theta} \log\left( \frac{g(h(\varepsilon,\theta)| \theta)}{r(h(\varepsilon,\theta)| \theta)}\right) \right]}_{g_{cor}}, \nonumber 
\end{align}
where $g_{rep}$ is the reparameterization term and $g_{cor}$ the correction term.
Since $h$ is invertible in $\varepsilon$, \citet{rejection-repar} show that $\nabla_\theta \log \dfrac{q(h(\varepsilon, \theta), \theta)}{r((h(\varepsilon, \theta), \theta)}$ in $g_{cor}$ simplifies to:
\begin{align} \label{h-invert-simple}
    \nabla_\theta \log \dfrac{g(h(\varepsilon, \theta), \theta)}{r((h(\varepsilon, \theta), \theta)} = \nabla_\theta \log g(h(\varepsilon, \theta), \theta) + \nabla_\theta \log |\dfrac{\partial h(\varepsilon, \theta)}{\partial\varepsilon}|,
\end{align}

\subsection{GRADIENT CALCULATION}
In our specific case we want to take the gradient w.r.t. $\theta$ of the expression:
\begin{equation}
    \E_{ q_\psi(\z|\x; \theta)} [\log p_\phi(\x | \z)] \quad \text{where } \theta = {(\mu,\kappa)},
\end{equation}
The gradient can be computed using the Lemma \ref{our-repar} and the subsequent gradient derivation with $f(\z) = p_{\phi}(\x|\z)$. 
As specified in Section \ref{sec:reparameterization} we optimize unbiased Monte Carlo estimates of the gradient. Therefore fixed one datapoint $\x$ and sampled $(\varepsilon,\vv)\sim\pi_1(\varepsilon| \theta)\pi_2(\vv)$ the gradient is:

\begin{align}
    \nabla_\theta \E_{ q_\psi(\z|\x; \theta)} [\log p_\phi(\x | \z)] = g_{rep} + g_{cor},   
\end{align}
With
\begin{align}
g_{rep} \approx \nabla_{\theta} \log p_{\phi}\left(\x|\mathcal{T}(h(\varepsilon, \theta), \vv; \theta)\right),
\end{align}
\begin{align}
    g_{cor} \approx p_{\phi}\left(x|\mathcal{T}(h(\varepsilon, \theta), \vv; \theta)\right) \left( \nabla_\theta \log g(h(\varepsilon, \theta)| \theta) + \nabla_\theta \log |\dfrac{\partial h(\varepsilon, \theta)}{\partial\varepsilon}| \right),
\end{align}

where $g_{rep}$ is simply the gradient of the reconstruction loss w.r.t $\theta$ and can be easily handled by automatic differentiation packages. 

For what concerns $g_{cor}$ we notice that the terms $g()$ and $h()$ do not depend on $\mu$. Thus the $g_{cor}$ term w.r.t.~$\mu$ is $0$ an all the following calculations can will be only w.r.t.~$\kappa$. We therefore have:

\begin{align}
\dfrac{\partial h(\varepsilon, k)}{\partial\varepsilon}=\dfrac{-2b}{((b-1) \varepsilon + 1)^2}    \quad \text{where }b = \dfrac{-2k + \sqrt{4k^2 + (m-1)^2}}{m - 1},
\end{align}
and
\begin{align}
\nabla_\kappa \log g(\omega| k) &= \nabla_\kappa \left( \log \mathcal{C}_m(k) + \omega k + \frac{1}{2}(m-3) \log (1-\omega^2) \right) \\
&= \nabla_k  \log \mathcal{C}_m(k)  + \nabla_\kappa \left( \omega k + \frac{1}{2}(m-3) \log (1-\omega^2) \right).
\end{align}
So, putting everything together we have:
\begin{align}
    g_{cor} &= \log p_\phi(x | z) \cdot  \left[  - \frac{\mathcal{I}_{m/2}(k)}{\mathcal{I}_{m/2 - 1}(k)}  + \nabla_\kappa \left( \omega \kappa + \frac{1}{2}(m-3) \log (1-\omega^2) + \log |\dfrac{-2b}{((b-1) \varepsilon + 1)^2}| \right) \right],
\end{align}
where
\begin{align}
b &= \dfrac{-2k + \sqrt{4k^2 + (m-1)^2}}{m - 1}\\
\omega &= h(\varepsilon,\theta) = \dfrac{1 - (1 + b)\varepsilon}{1 - (1 - b)\varepsilon} \\
z &= \mathcal{T}(h(\varepsilon, \theta), \vv; \theta),
\end{align}
And the term $\nabla_\kappa \left( \omega \kappa + \frac{1}{2}(m-3) \log (1-\omega^2) + \log |\dfrac{-2b}{((b-1) \varepsilon + 1)^2}| \right)$ can be computed by automatic differentiation packages. 

\section{COLLAPSE OF THE SURFACE AREA}\label{app:surface}

\begin{figure}[H]
  \centering
  \includegraphics[width=\textwidth]{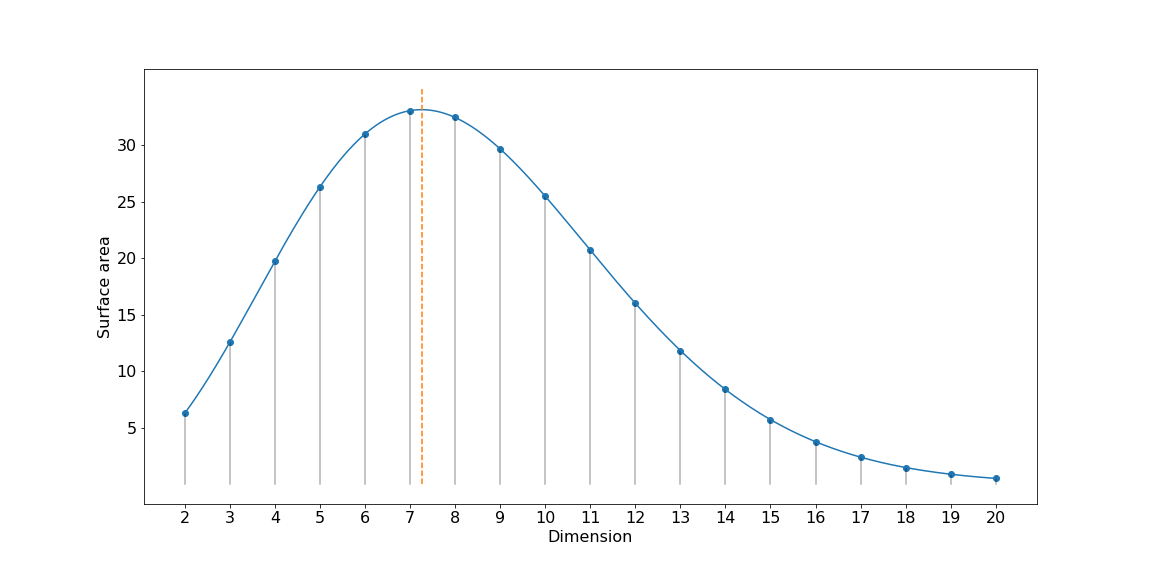}
  \caption{Plot of the unit hyperspherical surface area against dimensionality. The surface area has a maximum for $m=7$.}
  \label{fig:sphere-area}
\end{figure}

\section{EXPERIMENTAL DETAILS: ARCHITECTURE AND HYPERPARAMETERS}
\label{sect:appDetails}

\subsection{EXPERIMENT 5.2}
\paragraph{Architecture and hyperparameters} For both the encoder and the decoder we use MLPs with 2 hidden layers of respectively, [256, 128] and [128, 256] hidden units. We trained until convergence using early-stopping with a look ahead of 50 epochs. We used the Adam optimizer \citep{kingma2014adam} with a learning rate of 1e-3, and mini-batches of size 64. Additionally, we used a linear \textit{warm-up} for 100 epochs \citep{bowman2015generating}. The weights of the neural network were initialized according to \citep{glorot2010understanding}.

\subsection{EXPERIMENT 5.3}
\paragraph{Architecture and Hyperparameters}
For M1 we reused the trained models of the previous experiment, and used $K$-nearest neighbors ($K$-NN) as a classifier with $k=5$. In the $\mathcal{N}$-VAE case we used the Euclidean distance as a distance metric. For the \Sv-VAE the geodesic distance $\arccos(\x^{\top} \y)$ was employed. The performance was evaluated for $N = [100, 600, 1000]$ observed labels.

The stacked M1+M2 model uses the same architecture as outlined by \cite{kingma-semi-super}, where the MLPs utilized in the generative and inference models are constructed using a single hidden layer, each with 500 hidden units. The latent space dimensionality of $\z_1$, $\z_2$ were both varied in $[5, 10, 50]$. We used the rectified linear unit (ReLU) as an activation function. Training was continued until convergence using early-stopping with a look ahead of 50 epochs on the validation set. We used the Adam optimizer with a learning rate of 1e-3, and mini-batches of size 100. All neural network weight were initialized according to \citep{glorot2010understanding}. $N$ was set to 100, and the $\alpha$ parameter used to scale the classification loss was chosen between $[0.1, 1.0]$. Crucially, we train this model end-to-end instead of by parts.

\subsection{EXPERIMENT 5.4}\label{ap:link-pred}
\paragraph{Architecture and Hyperparameters} We are training a Variational Graph Auto-encoder (VGAE) model, a state-of-the-art link prediction model for graphs, as proposed in \citet{kipf2016VGAE}. For a fair comparison, we use the same architecture as in the original paper and we just change the way the latent space is generated using the vMF distribution instead of a normal distribution. All models are trained for 200 epochs on Cora and Citeseer, and 400 epochs on Pubmed with the Adam optimizer. Optimal learning rate $lr\in\{0.01, 0.005, 0.001\}$, dropout rate $p_{do}\in \{0, 0.1, 0.2, 0.3, 0.4\}$ and number of latent dimensions $d_z\in \{8,16,32,64\}$ are determined via grid search based on validation AUC performance. For \Sv-VGAE, we omit the $d_z=64$ setting as some of our experiments ran out of memory. The model is trained with a single hidden layer with $32$ units and with document features as input, as in \citet{kipf2016VGAE}. The weights of the neural network were initialized according to \citep{glorot2010understanding}. For testing, we report performance of the model selected from the training epoch with highest AUC score on the validation set. Different from \citep{kipf2016VGAE}, we train both the \Nv-VGAE and the \Sv-VGAE models using negative sampling in order to speed up training, i.e.~for each positive link we sample, uniformly at random, one negative link during every training epoch. All experiments are repeated 5 times, and we report mean and standard error values.

\subsubsection{FURTHER EXPERIMENTAL DETAILS}
Dataset statistics are summarized in Table \ref{tab:link-datasets}. Final hyperparameter choices found via grid search on the validation splits are summarized in Table \ref{tab:link-params}.

\begin{table}[htp]
\centering
\caption{\label{tab:link-datasets}Dataset statistics for citation network datasets.}
\bigskip
\begin{tabular}{l r r r}
\toprule
Dataset & Nodes & Edges  & Features \\[0.05em]\hline \\[-0.8em]
\textbf{Cora} & 2,708 & 5,429 & 1,433 \\
\textbf{Citeseer} & 3,327 & 4,732 & 3,703 \\
\textbf{Pubmed} & 19,717 & 44,338 & 500 \\
\bottomrule
\end{tabular}
\end{table}

\begin{table}[htp]
\centering
\caption{\label{tab:link-params}Best hyperparameter settings found for citation network datasets.}
\bigskip
\begin{tabular}{l c c c c}
\toprule
Dataset & Model & $lr$ & $p_{do}$  & $d_z$ \\[0.05em]\hline \\[-0.8em]
\multirow{2}{0pt}{\textbf{Cora}}     & $\mathcal{N}$-VAE & 0.005  & 0.4 & 64 \\
                                  & $\mathcal{S}$-VAE & 0.001 & 0.1 & 32 \\ \midrule
\multirow{2}{0pt}{\textbf{Citeseer}} & $\mathcal{N}$-VAE & 0.01  & 0.4 & 64 \\
                                  & $\mathcal{S}$-VAE & 0.005 & 0.2 & 32 \\ \midrule
\multirow{2}{0pt}{\textbf{Pubmed}}   & $\mathcal{N}$-VAE & 0.001 & 0.2 & 32 \\
                                  & $\mathcal{S}$-VAE & 0.01 & 0.0 & 32 \\
                                \bottomrule
\end{tabular}
\end{table}

\newpage
\section{VISUALIZATION OF SAMPLES AND LATENT SPACES} \label{app:mnist-latent-visual}

\begin{figure*}[h!]
\centering
     \subfigure[$d=2$]{\includegraphics[width=0.24\textwidth]{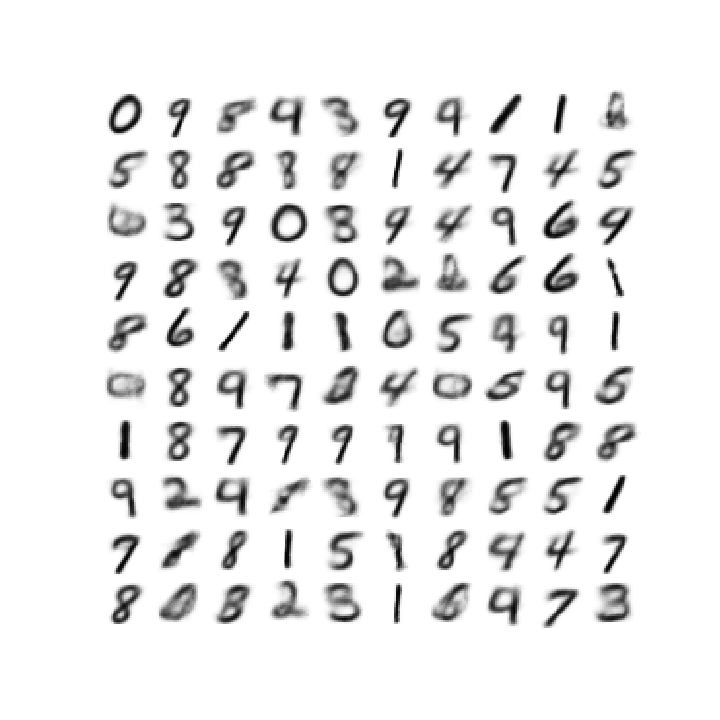} \label{fig:r2sample}}
     \subfigure[$d=5$]{\includegraphics[width=0.24\textwidth]{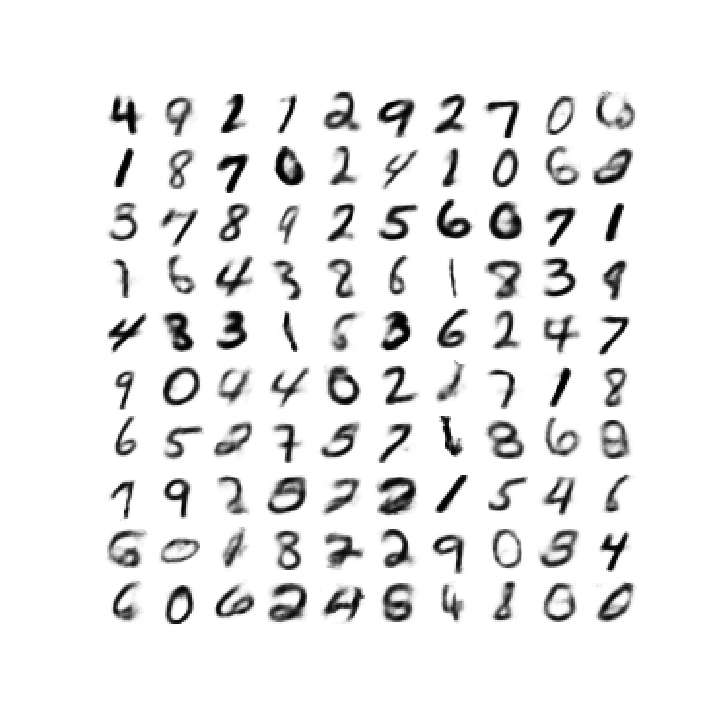} \label{fig:r5sample}}
     \subfigure[$d=10$]{\includegraphics[width=0.24\textwidth]{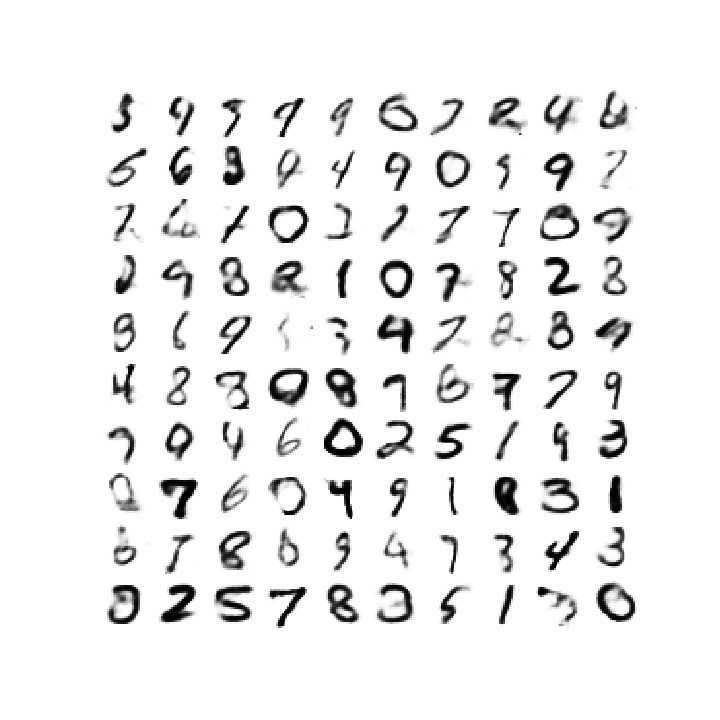} \label{fig:r10sample}}
     \subfigure[$d=20$]{\includegraphics[width=0.24\textwidth]{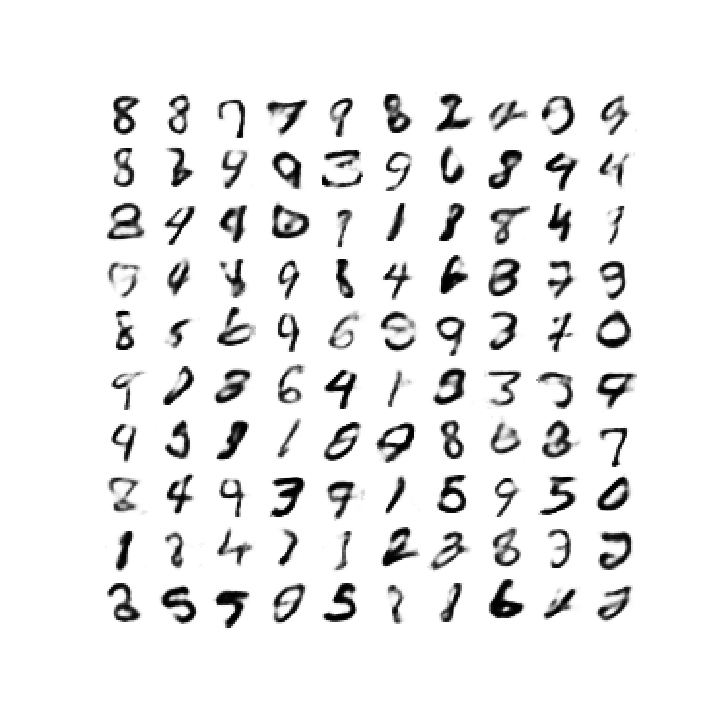} \label{fig:r20sample}}
\caption{Random samples from $\mathcal{N}$-VAE of MNIST for different dimensionalities of latent space.}
\label{fig:nvae_samples}
\end{figure*}

\begin{figure*}[h!]
\centering
     \subfigure[$d=2$]{\includegraphics[width=0.24\textwidth]{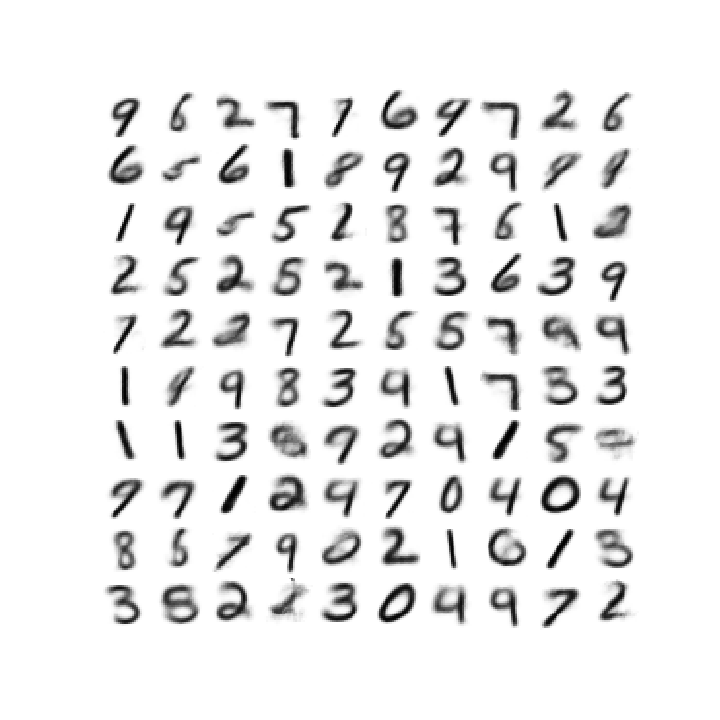} \label{fig:s2sample}}
     \subfigure[$d=5$]{\includegraphics[width=0.24\textwidth]{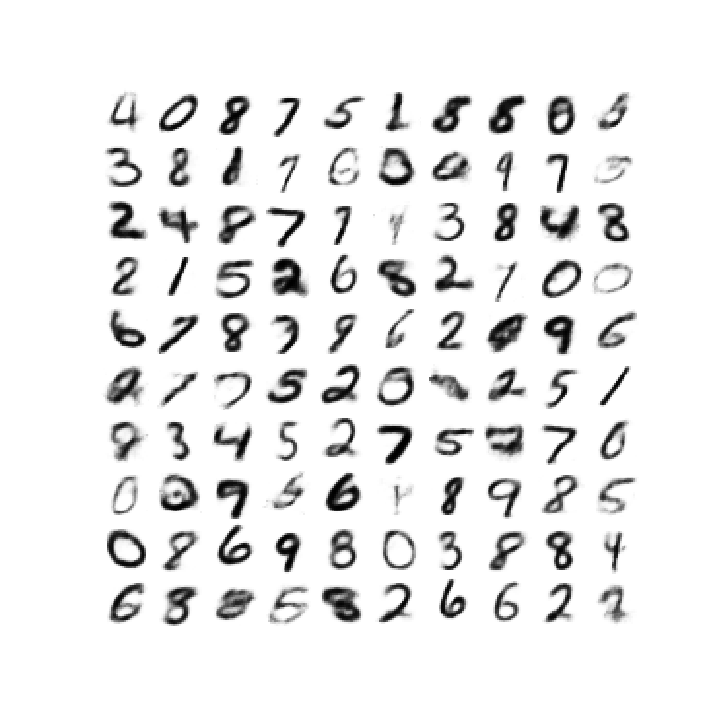} \label{fig:s5sample}}
     \subfigure[$d=10$]{\includegraphics[width=0.24\textwidth]{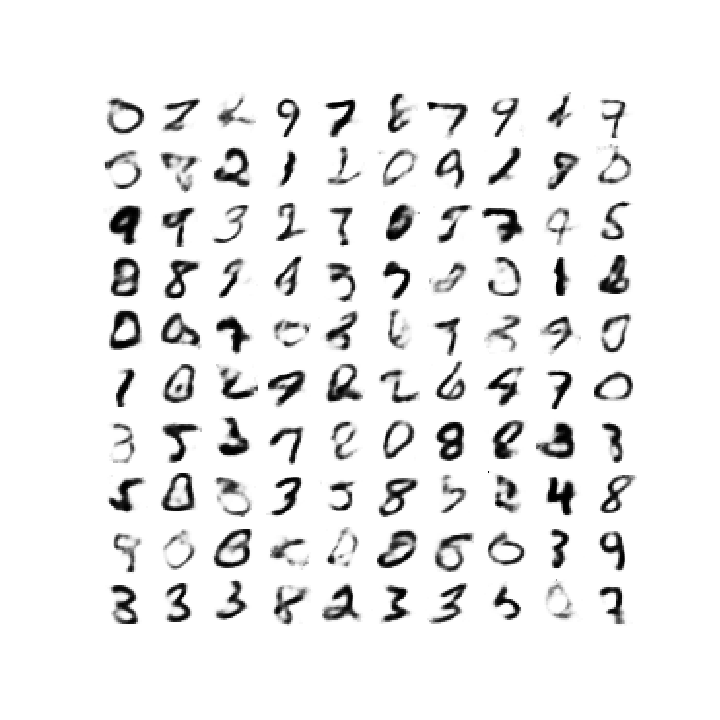} \label{fig:s10sample}}
     \subfigure[$d=20$]{\includegraphics[width=0.24\textwidth]{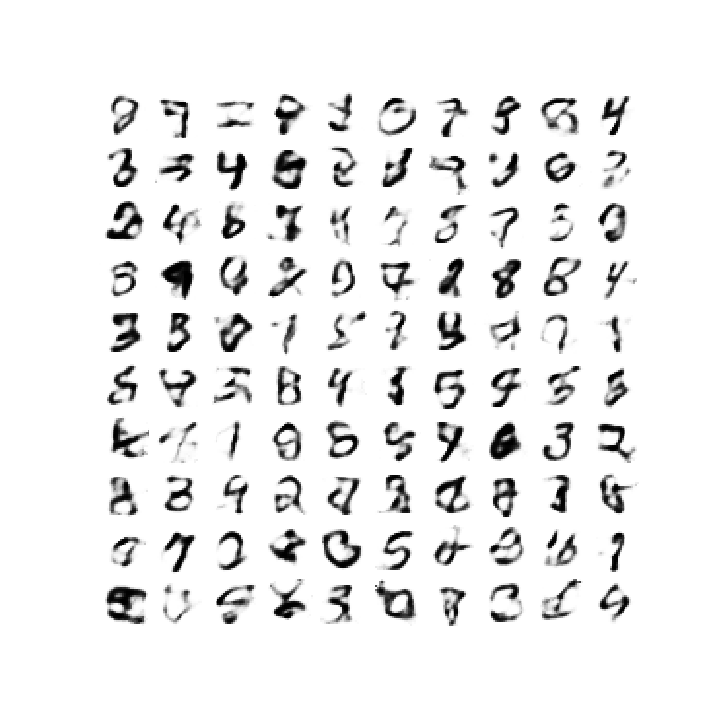} \label{fig:s20sample}}
\caption{Random samples from $\mathcal{S}$-VAE of MNIST for different dimensionalities of latent space.}
\label{fig:svae_samples}
\end{figure*}

\begin{figure*}[h!]
\centering
     \subfigure[$\mathcal{N}$-VAE]{\includegraphics[width=0.45\textwidth]{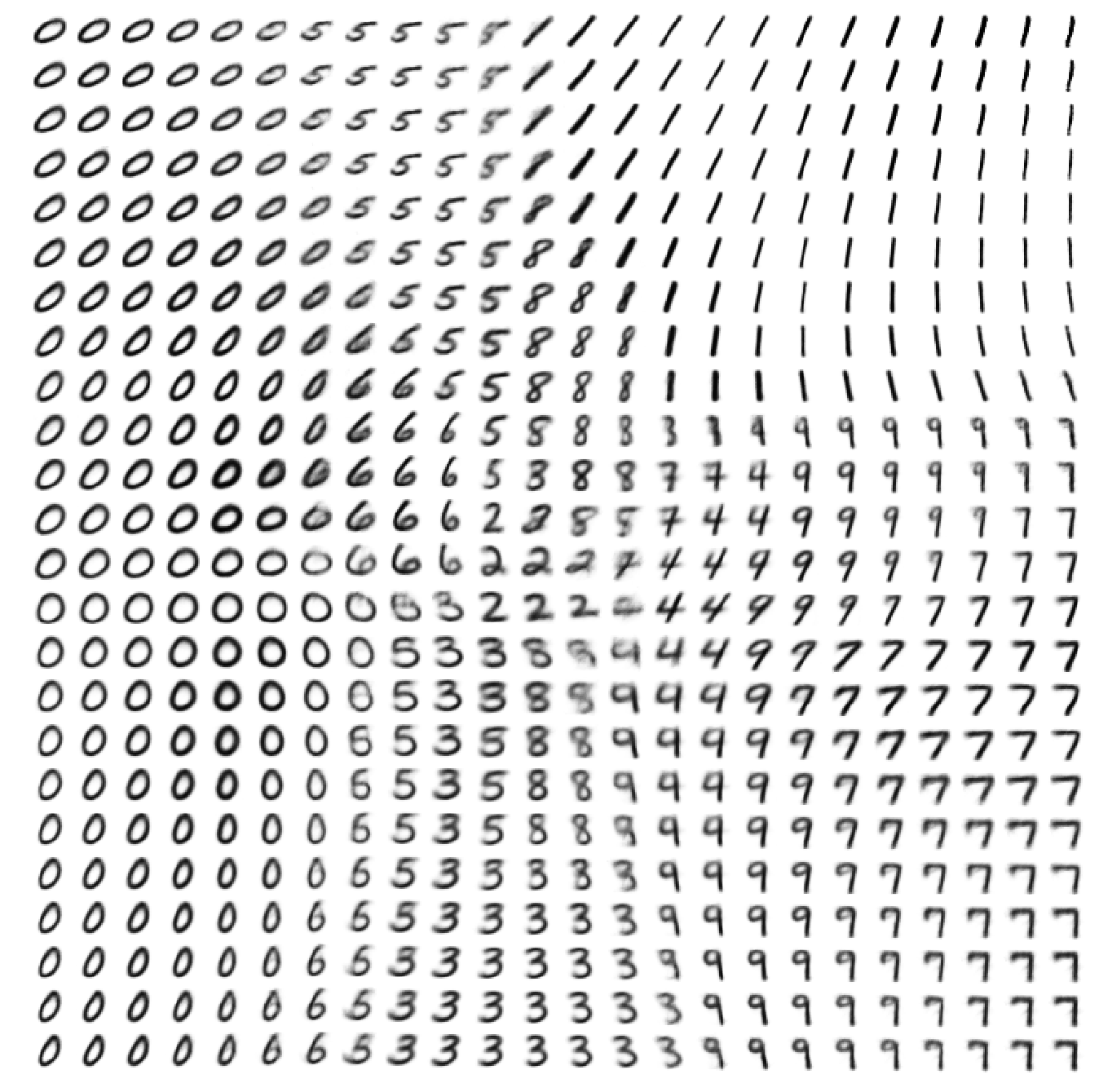} \label{fig:r2latent}}
     \subfigure[$\mathcal{S}$-VAE]{\includegraphics[width=0.49\textwidth]{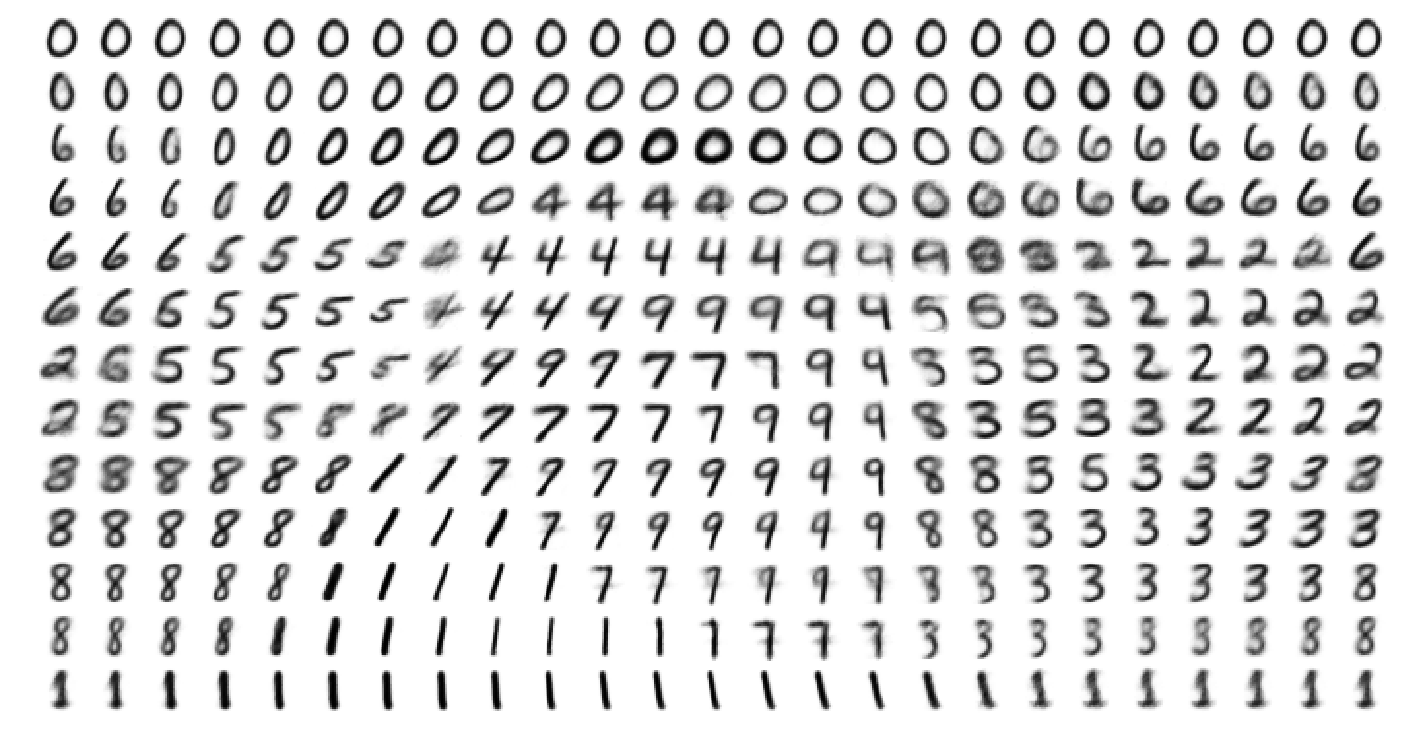} \label{fig:s2latent}}
\caption{Visualization of the 2 dimensional manifold of MNIST for both the \Nv-VAE and \Sv-VAE. Notice that the \Nv-VAE has a clear center and all digits are spread around it. Conversely, in the \Sv-VAE instead all digits occupy the entire space and there is a sense of continuity from left to right.}
\label{fig:latent_spaces_appendix}
\end{figure*}

\vfill
\newpage
\section{VISUALIZATION OF CONDITIONAL GENERATION} \label{app:cond-visual}
\begin{figure*}[h!]
\centering
\includegraphics[width=0.3\textwidth]{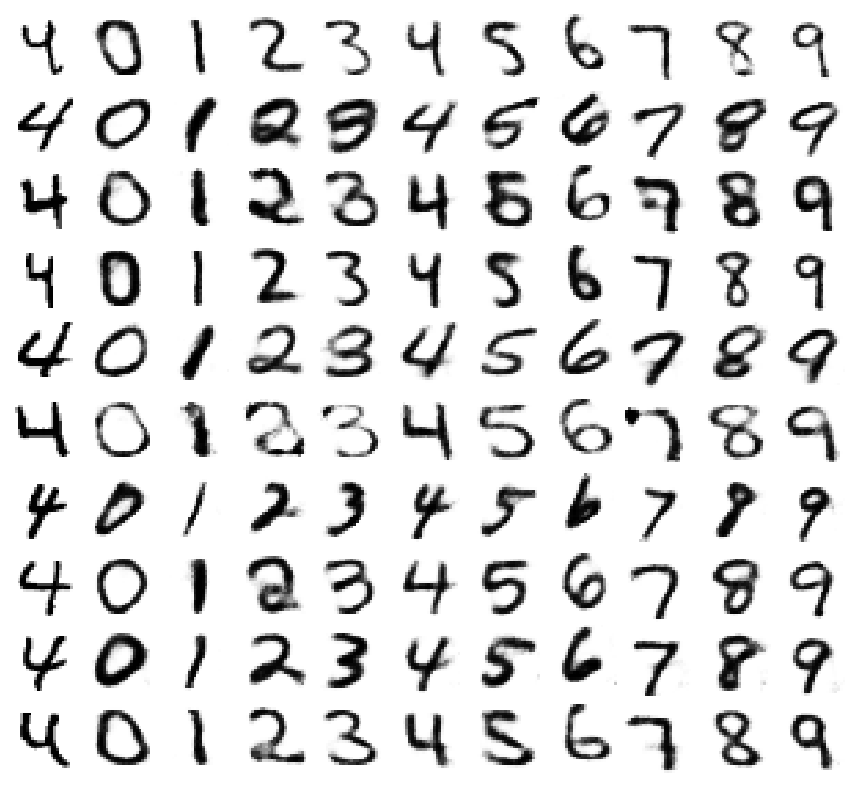}
\caption{Visualization of handwriting styles learned by the model, using conditional generation on MNIST of M1+M2 with $dim(\z_1) = 50$, $dim(\z_2)=50$, \Sv+\Nv. Following \citet{kingma-semi-super}, the left most column shows images from the test set. The other columns show analogical fantasies of $\x$ by the generative model, where in each row the latent variable $\z_2$ is set to the value inferred from the test image by the inference network and the class label $\y$ is varied per column.}
\label{fig:latent_spaces}
\end{figure*}

\vfill
\end{document}